\documentclass[sigconf]{aamas} 

\usepackage{balance} 

\hypersetup{urlcolor=blue, colorlinks=true}

\usepackage{mathtools}

\usepackage{multirow}
\usepackage{hhline}
\usepackage{footnote}
\newcommand{\comments}[1]{{\color{blue}\textit{$\#$ #1}}}

\newcommand{\cA}{{\mathcal{A}}}

\newcommand{\cS}{{\mathcal{S}}}

\newcommand{\cH}{{\mathcal{H}}}

\newcommand{\cO}{{\mathcal{O}}}

\newcommand{\bE}{{\mathbb{E}}}

\newcommand{\bq}{\textbf{q}}

\newcommand{\br}{\textbf{r}}

\newcommand{\bbE}{\mathbb{E}}

\newcommand{\KL}{\textsc{KL}}
\newcommand{\AV}{{\tiny\textsc{$\alpha$}}}
\newcommand{\VT}{{\tiny\textsc{$\nu$}}}

\usepackage{mathtools}

\usepackage{mathtools}

\newif\ifnotes\notestrue
\def\htien#1{}

\usepackage{graphicx}
\usepackage{algorithm}
\usepackage{algpseudocode}
\usepackage{subcaption}
\usepackage{multirow}
\usepackage{color, colortbl}
\definecolor{Gray}{gray}{0.92}
\newcolumntype{g}{>{\columncolor{Gray}}c}

\newcommand{\showinstance}[4][0.25]{
\begin{subfigure}{#1\textwidth}
  \centering
  \includegraphics[width=1.0\linewidth]{graphs/#2/#3.pdf}
  \ifx #4\empty \else \caption*{#4} \fi
  \captionsetup{justification=centering}
\end{subfigure}
}
\newcommand{\showtsne}[4][0.25]{
\begin{subfigure}{#1\textwidth}
  \centering
  \includegraphics[width=1.0\linewidth, trim={0.65cm 0.2cm 0.65cm 0.2cm},clip]{graphs/#2/#3.pdf}
  \ifx #4\empty \else \caption*{#4} \fi
  \captionsetup{justification=centering}
\end{subfigure}
}
\newcommand{\showtsnapshot}[2][0.25]{
\begin{subfigure}{#1\textwidth}
  \centering
  \raisebox{-.4\totalheight}{\includegraphics[width=1.0\linewidth]{graphs/Snapshots/#2.jpg}}
  \captionsetup{justification=centering}
\end{subfigure}
}
\newcommand{\red}[1]{\textcolor{red}{\textbf{#1}}}

\acmConference[]{Preprint}{}
{}{}
\copyrightyear{2022}
\acmYear{2022}
\acmDOI{}
\acmPrice{}
\acmISBN{}

\acmSubmissionID{???}

\title[APIL]{Imitating Opponent to Win: Adversarial Policy Imitation Learning in Two-player Competitive Games}

 \author{The Viet Bui}
 \affiliation{
   \institution{Singapore Management University}
   \city{Singapore}
   \country{Singapore}}
 \email{tvbui@smu.edu.sg}

 \author{Tien Mai}
 \affiliation{
   \institution{Singapore Management University}
 \city{Singapore}
 \country{Singapore}}
 \email{atmai@smu.edu.sg}

\author{Thanh H.Nguyen}
\affiliation{
	\institution{University of Oregon}
	\city{Eugene, Oregon}
	\country{United States}}
\email{thanhhng@cs.uoregon.edu}

\begin{abstract}

Recent research on vulnerabilities of deep reinforcement learning (RL) has shown that adversarial policies adopted by an adversary agent can influence a target RL agent (victim agent) to perform poorly in a multi-agent environment. In existing studies, adversarial policies are directly trained based on experiences of interacting with the victim agent. There is a key shortcoming of this approach ---  knowledge derived from historical interactions may not be properly generalized to unexplored policy regions of the victim agent, making the trained adversarial policy significantly less effective. In this work, we design a new effective adversarial policy learning algorithm that overcomes this shortcoming. The core idea of our new algorithm is to create a new imitator --- the imitator will learn to imitate the victim agent's policy while the adversarial policy will be trained not only based on interactions with the victim agent but also based on feedback from the imitator to forecast victim's intention. By doing so, we can leverage the capability of imitation learning in well capturing underlying characteristics of the victim policy only based on sample trajectories of the victim. 
Our victim imitation learning model differs from prior models as the environment's dynamics are driven by adversary's policy and will keep changing during the adversarial policy training. We provide a provable bound to guarantee a desired imitating policy when the adversary's policy becomes stable. 
We further strengthen our adversarial policy learning by making our imitator a stronger version of the victim. That is, we incorporate the opposite of the adversary's value function to the imitation objective, leading the imitator not only to learn the victim policy but also to be adversarial to the adversary. Finally, our extensive experiments using four competitive MuJoCo game
environments show that our proposed adversarial policy learning algorithm outperforms state-of-the-art algorithms. 
\end{abstract}


\newcommand{\BibTeX}{\rm B\kern-.05em{\sc i\kern-.025em b}\kern-.08em\TeX}

\begin{document}


\pagestyle{fancy}
\fancyhead{}


\maketitle 


\section{Introduction}
Exploring vulnerabilities of deep reinforcement learning has drawn a lot of interests from the AI research community~\cite{ma2019policy,rakhsha2020policy,behzadan2017vulnerability,rakhsha2021reward}, given recent successes of deep RL in accomplishing a variety of interesting multi-agent learning tasks~\cite{dosovitskiy2017carla,lewis2017deal,noonan2017jpmorgan,nazari2018reinforcement}. Most of the existing work follows the traditional adversarial learning framework which often makes strong assumptions about the adversary's capabilities. Typically, the attacker is assumed to be able to manipulate input image observations or even interfere with the learning process of the victim agent. As pointed out in some recent work~\cite{Gleave2020AdversarialPA}, these assumptions are not practical, especially in real-world domains such as autonomous driving in which the attacker cannot easily modify the input of the victim policy. 

A recent alternative attack approach to the victim policy was introduced in \cite{Gleave2020AdversarialPA} which presents the idea of \emph{adversarial policy}. Essentially, an attacker can build an adversarial policy for an opponent agent (this agent is under control by the attacker, thus is called adversary agent in our paper) that takes actions in a shared environment with the victim. This adversarial policy can weaken the outcome of the victim's policy, not by making the opponent choose stronger actions, but instead by inducing natural observations that can lead the victim to behave in an undesired way. Following this interesting approach, Guo et al. \cite{Guo2021AdversarialPL} extended the two-player zero-sum Markov game model used in \cite{Gleave2020AdversarialPA} for the general non-zero-sum game setting and introduced a more effective adversarial policy learning algorithm based on a surrogate learning objective function. 

Motivated by these initial successes which demonstrate negative effects of adversarial policy on the victim's policy, our paper focuses on designing new stronger adversarial policies in the non-zero-sum Markov game setting. Our key contribution is to introduce an imitator of which goal is to discover intrinsic properties of the victim's policy from historical interactions between the victim and adversary agents --- this knowledge is then transferred to the attacker to improve the adversarial policy. By following this idea, we are able to exploit the advantage of imitation learning \citep{ho2016generative} to anticipate the victim agent's moves in unexplored policy regions, allowing us to strengthen the generated adversarial policy. This is a significant advancement compared to previous work~\cite{Gleave2020AdversarialPA,Guo2021AdversarialPL} which learns an adversarial policy directly from past interactions with the victim agent, substantially limiting the impact of the trained adversarial policy on unseen policy regions of the victim agent. 

A key challenge in incorporating our new imitator into the adversarial policy learning framework is that the imitator is trained simultaneously with the adversarial policy learning --- the inter-dependency between the imitator and the adversary agent complicates the entire training process. Note that in traditional imitation learning, the set of the victim's policy trajectories used for training the imitator is fixed during the training process. On the other hand, in the context of adversarial policy learning, the environment's dynamics (and as a result, the victim trajectories obtained for training the imitator) are driven by the adversarial policy. The adversarial policy, in turn, is trained based on the policy output of the imitator. This learning inter-dependency complication requires a carefully designed training process to ensure a convergence to high-quality outcomes for both the imitation policy and the adversarial policy. 

To address this learning challenge, we provide the following contributions: (i) we theoretically characterize the dependency between the adversarial policy and the imitation policy; (ii) we provide a provable bound to guarantee a desired imitating policy when the adversary's policy becomes stable; and (iii)
we further strengthen our adversarial policy learning by enhancing our imitator --- we incorporate the negation of the adversary's goal into the imitator's objective function, making the imitator to both learn the victim policy and be adversarial to the adversary agent. 

Lastly, we conduct extensive experiments on four competitive MuJoCo game environments introduced by Emergent Complexity \cite{Bansal2018EmergentCV} (including Kick And Defend, You Shall Not Pass, Sumo Humans, and Sumo Ants) to evaluate our proposed adversarial policy learning algorithms. We empirically show that our generated adversarial policies obtained a significantly higher winning (plus tie) rates against the victim agent in these game environments, in comparison with  state-of-the-art adversarial policy methods. We further show that we can make the victim agent substantially more resilient to any adversarial policies (generated by different algorithms) by retraining the victim policy against our adversarial policies. 
\section{Related Work}
\paragraph{Attacks on deep RL} Existing works have focused on creating attacks by manipulating  victim's observations or victim's actions. For example,  \cite{Huang2017AdversarialAO} and \cite{behzadan2017vulnerability} propose to perturb victim's observations to force the victim to make sub-optimal actions, thus fails the task. Later works \citep{kos2017delving,russo2019optimal, lin2017tactics, sun2020stealthy,zhang2021robust} extend this approach by proposing to manipulate  victim's observations at some selected time steps, instead of the whole trajectories. Other papers \citep{Huang2017AdversarialAO,behzadan2017vulnerability,zhao2020blackbox,xiao2019characterizing,lin2020robustness} focus on attacks through observation manipulation but in black-box settings, i.e.,  the adversary does not have the power to manipulate victim's observations, but can access the input and output of the victim's policy or deep-Q networks. Besides, there are works that propose to directly perturb actions taken by victim agents in both white-box and black-box settings 
\citep{xiao2019characterizing,lee2020spatiotemporally}. There is another emerging line of approaches that view the attacks as a two-player competitive game, i.e.,  focusing on training an adversarial agent to play with the victim.  For example, \cite{Gleave2020AdversarialPA} train an adversarial agent using  Proximal Policy Optimization (PPO) \citep{schulman2017proximal} for a set of MuJoCo game environments \citep{todorov2012mujoco}, and \cite{Guo2021AdversarialPL} propose to break the zero-sum-game setting  and redesign the adversary's reward function to  achieve better adversary agents. Our methods belong to this direction,
but differ from prior methods as we create an imitator that uses imitation learning to mimic and predict victim's intention --- this prediction can be used to further strengthen the adversarial agent training.   

\paragraph{Imitation learning}
A core component of  our  algorithms is an imitation learning model trained to mimic and product similar victim's actions. We employed an extended version of the Generative Adversarial Imitation Learning (GAIL) algorithm \citep{ho2016generative,song2018multi}, a state-of-the-art imitation leaning algorithm that is highly scalable for continuous domains such as the MuJoCo ones. The literature on learning from expert demonstrations covers both  imitation learning  and inverse reinforcement learning (IRL) works. While imitation learning presents a direct approach to imitate expert's policies, IRL \citep{finn2016connection,finn2016guided,fu2017learning,yu2019multi} assumes that expert's policy is driven by an expert's reward function, thus propose to infer this function from demonstrations. Even-though IRL is more transferable for changing environments, it is often less effective  in mimicking expert's demonstrations \citep{ho2016generative}. On the other hand, for the MuJoCo game environments, since the rewards are obvious, imitation learning (or specifically GAIL) provides us with a direct and suitable algorithm for imitating the victim. Note that, in our settings, the victim's environmental dynamics keep changing during the adversarial training, raising a need for redesigning the GAIL, in both theoretical and practical aspects. We address this issue later in this paper.

    




\section{Adversarial Policy Framework}
Following the general framework introduced in \cite{Gleave2020AdversarialPA}, we consider a two-player Markov game in which the victim plays against an opponent which is under control by the adversary. We thus name these two players the victim agent and adversary agent. We represent the two-player non-zero-sum Markov game as a tuple: 
\begin{eqnarray*}
(\cS,
 \cA^{\AV}, \cA^{\VT},\bq^{\AV}, \bq^{\VT}, \br^{\AV}, \br^{\VT},\gamma),
\end{eqnarray*}
where $\alpha$ refers to the adversary and $\nu$ refers to the victim. In addition,
  $\cS$ is the  set of the states, $(\cA^{\AV},\cA^{\VT})$ are sets of actions, 
 $(\bq^{\AV},\bq^{\VT})$ are transition probabilities  and  $(\br^{\AV}, \br^{\VT})$ are reward functions  of the two players, respectively. Finally, $\gamma\in[0,1]$ is a discount factor. 
 The objective of each player is to maximize his/her long-term expected reward. 
Essentially, given a policy of the victim, denoted by $\pi^{\VT}$, the adversary aims at finding an optimal policy $\pi^{\AV}$ that maximizes their long-term reward, formulated as follows:
 \[
 \max_{\pi^{\AV}} \left\{V_{\pi^{\AV}}(s_0|\bq^\AV(\pi^{\VT})) = \bE_{\tau \sim \pi^{\AV}}\left[\sum_{t=0}^\infty \gamma^t r^\AV(s_t)\;\Big| \bq^\AV(\pi^{\VT}) \right] \right\},
 \]
where $s_0$ is the initial state and $\tau$ denotes a trajectory sampled from executing the adversary policy $\pi^{\AV}$ in the environment and $s_t\in \tau$ for all $t$. Transition probabilities of the adversary, denoted by $\bq^\AV(\pi^{\VT})$, depends on the policy of the victim  $\pi^\VT$. Specifically, the transition probability $q^\AV(s_{t+1} |s_{t}, a^\AV_t)$ can be generally computed as follows:
\begin{align*}
q^\AV(s_{t+1} |s_{t}, a^\AV_t)  
= \sum_{a^{\VT} \in \cA^{\VT}} \pi^{\VT} (a^{\VT}_{t}|s_{t}) P(s_{t+1}| s_{t},a^{\AV}_t, a^{\VT}_{t}),
\end{align*}
where  $P(s_{t+1}| s_{t},a^{\AV},a^{\VT}_{t})$ is the probability of reaching  state   $s_{t+1}$ if the  adversary and victim agents take action $a^{\AV}_{t}, a^{\VT}_{t}$, respectively, at state $s_{t}$. 
In other words, if the policy of the victim is fixed, then the transition probabilities $\bq^{\AV}$ are also fixed. Similarly, the objective of the victim is to maximize the expected long-term rewards of the victim, formulated as follows:
 \begin{align*}
 \max_{\pi^{\VT}}\left\{ V_{\pi^{\VT}}(s_0| \bq^\VT(\pi^{\AV}))  = \bE_{\tau \sim \pi^{\VT}}\left[\sum_{t=0}^\infty \gamma^t r^{\VT}(s_t)\;\Big| \bq^\VT(\pi^{\AV}) \right] \right\}.
 \end{align*}
Intuitively, if the policy of one player is fixed, then the transition probabilities (or dynamics) of the other player's environment is also fixed; thus the two-player game becomes a standard RL task. 

In this paper, we assume the victim follows a fixed policy (which was pre-trained). This is a common assumption in adversarial policy learning research, motivated by real-world settings such as autonomous vehicles in which RL-trained policies might be deployed~\cite{Guo2021AdversarialPL}. We will later discuss the effect of unfixed victim's policies on the adversarial policy training.
As mentioned previously, existing work directly trains adversarial policies based on interactions with the victim. We instead create an imitator who follows imitation learning to discover underlying characteristics of the victim policy and transfers that knowledge to the adversary, helping the adversary produces a better adversarial policy. Both the imitator and the adversary policies will be trained simultaneously based on interactions between the adversary and the victim.  

Next, we will first present our imitation learning of the victim policy. We then follow with the elaboration on our adversarial policy learning that incorporates the victim imitation learning component.

\section{Victim Imitation Learning}
In standard imitation learning, we learn to imitate an expert (which is the victim in our study) based on a fixed set of trajectories sampled from the expert's policy. On the other hand, in our problem, learning the victim policy is more challenging since it involves the adversarial policy which is also being trained at the same time (our observations of the victim policy depend on what policy the adversary is playing). In the following, we first introduce our advanced imitation model and algorithm given a fixed adversary policy. We then present our theoretical results on the impact of the adversary policy (during the training process) on our imitation learning.  
\subsection{Enhanced Imitation Learning Model}
Our objective is to build an imitation learning model to imitate the victim's policy through observing victim trajectories. Our model is essentially an enhanced version of the GAIL algorithm \citep{ho2016generative}. 
Overall, following GAIL, given an attacker policy $\pi^{\AV}$, an imitation policy can be learned by solving the following saddle point problem:
\begin{align}
&\max_{\widetilde{\pi}^{\VT}_\psi} \; \min_{D_w}  \Bigg\{\phi(\widetilde{\pi}^{\VT}_\psi,D_w) = \bbE_{\tau \sim \widetilde{\pi}^{\VT}_\psi} \left[\sum\nolimits_{t}\log (D_w(s_t,a^\VT_t)) ~\big| \bq^\VT(\pi^{\AV})\right]\nonumber \\
&+ \bbE_{\tau \sim {\pi}^{\VT}} \left[\sum\nolimits_{t} \log (1 - D_w(s_t,a^\VT_t)) ~\big| \bq^\VT(\pi^{\AV})\right]  - \lambda H(\widetilde{\pi}^{\VT}_\psi) \Bigg\}\label{prob:AIL}
\end{align}
where $H(\cdot)$ is the entropy function, 
 $\widetilde{\pi}^{\VT}_\psi$ refers to the imitating  policy which is an output of a neural net with parameter $\psi$, and $\pi^\VT$ is the victim's policy to be imitated. In addition, $D$ is a discriminative neural net model with parameter $w$ to distinguish between trajectories generated by $\widetilde{\pi}^{\VT}_\psi$ and
those from the victim’s policy. Normally, GAIL would require a large amount of victim's trajectories to provide a good imitation policy. Since we want the imitation model to work with the adversary's policy optimization simultaneously, this is difficult to achieve at early  {episodes} when the set of demonstrations from the victim is limited. Therefore,
we propose to robustify GAIL by adding  the adversary's value function to the objective: 
\begin{align}{
\max_{\widetilde{\pi}^{\VT}_\psi} \; \min_{D_w} \Bigg\{\phi^E(\widetilde{\pi}^{\VT}_\psi,D_w\big|\pi^\AV) = \phi(\widetilde{\pi}^{\VT}_\psi,D_w) - V_{\pi^{\AV}}(s_0|\bq^\AV(\widetilde{\pi}^{\VT}_\psi))\Bigg\}\label{prob:E-GAIL}
}
\end{align}
In this enhanced model \eqref{prob:E-GAIL}, the aim to train a policy that both mimics the victim and minimizes the adversary's long-term reward.
The value function $V_{\pi^{\AV}}(s_0|\bq^\AV(\widetilde{\pi}^{\VT}_\psi))$ is the adversary's expected reward, but defined as a function of imitator's policy.
The inclusion of the opponent's value function makes \eqref{prob:E-GAIL} not straightforward to handle and would require a redesign of the objective function to make it practical, as stated  in \citep{Guo2021AdversarialPL}. Despite of that, we can provide, in the following, a simple formulation for the gradient of $V_{\pi^{\AV}}(s_0|\bq^\AV(\widetilde{\pi}^{\VT}_\psi))$, making it convenient to be handled by a standard policy optimization algorithm, e.g., TRPO or PPO \cite{schulman2015trust,schulman2017proximal}. 
\begin{lemma} \label{lm:lm1}
The gradient of the value function for the adversary $V_{\pi^{\AV}}(s_0|\bq^\AV(\widetilde{\pi}^{\VT}_\psi))$ w.r.t  $\psi$ can be computed as follows:
\begin{align*}
\nabla_\psi &\left(V_{\pi^{\AV}}(s_0|\bq^\AV(\widetilde{\pi}^{\VT}_\psi))\right) \\
& \qquad= \bbE_{\tau\sim \widetilde{\pi}^{\VT}_\psi} \left[ R^{\AV}(\tau)\sum\nolimits_{t} \nabla_{\psi}\log \widetilde{\pi}^{\VT}_\psi(a^{\VT}_t|s_t)\Big| \bq^\VT(\pi^\AV)) \right], 
\end{align*}
where $R^\AV(\tau) = \sum_t \gamma^t r^\AV(s_t)$ with $s_t\in \tau$. 
\end{lemma}
As a result, at each imitation learning step, after updating the discriminator $D_w(s_t,a^{\VT}_t)$, one can update the imitating policy $\widetilde{\pi}^{\VT}_\psi$ using the gradient given in Proposition \eqref{prop:IM-gradient} below.\footnote{Detailed proofs of all theoretical results are in the appendix.} 
\begin{proposition}\label{prop:IM-gradient}
The gradient of the objective \eqref{prob:E-GAIL} w.r.t. $\psi$ can be computed as follows: 
\begin{multline*}
\bbE_{\tau\sim \widetilde{\pi}^{\VT}_{\psi}}\left[\sum\nolimits_{t} \gamma^t \eta(s_t, a^{\AV}_t, a^{\VT}_t) \sum\nolimits_t \nabla_{\psi} \log \widetilde{\pi}^{\VT}_\psi(a^{\VT}_t|s_t) \right] 
- \lambda \nabla_{\psi} H(\widetilde{\pi}^{\VT}_{\psi})
\end{multline*}
where 
$
\eta(s_t, a^{\AV}_t, a^{\VT}_t) = \log (D_w(s_t, a^{\VT}_t)) - r^{\AV}(s_t).
$
\end{proposition}
With all the findings above, we can show that the enhanced imitation learning model \eqref{prob:E-GAIL}  can be converted into a standard GAIL with a modified objective function, with a note that the imitation learning model depends on the adversary's policy $\pi^\AV$, which dictates the dynamics of the victim's environment. 
\begin{corollary}
\label{coro:c1}
The enhanced imitation learning model \eqref{prob:E-GAIL} is equivalent to GAIL with the modified objective:
\begin{align}
\max_{\widetilde{\pi}^{\VT}} \; & \min_{D_w}  \Bigg\{\phi^E(\widetilde{\pi}^{\VT}_\psi,D_w|\pi^\AV) = \bbE_{\tau \sim \widetilde{\pi}^{\VT}_\psi} \left[\sum_{t} \eta(s_t, a^{\AV}_t, a^{\VT}_t)  ~\big| \bq^\VT(\pi^{\AV})\right]\nonumber \\
&+ \bbE_{\tau \sim {\pi}^{\VT}} \left[\sum_{t} \log (1 - D_w(s_t,a^{\VT}_t)) ~\big| \bq^\VT(\pi^{\AV})\right]  - \lambda H(\widetilde{\pi}^{\VT}_\psi) \Bigg\}
\end{align}
\end{corollary}

\subsection{Imitation Learning Algorithm}
\label{sec:IM}
With all the theoretical results on gradient computation developed in the previous section, we are now ready for the  victim imitation learning algorithm. As shown in Corollary \ref{coro:c1}, the enhanced imitation learning model can be converted to a standard one, implying that the same optimization steps in \cite{ho2016generative}  can be used with the the modified discriminator's objective. That is, at each iteration of the adversarial policy optimization, 
one can follow the following three steps to update the imitating policy $\widetilde{\pi}^{\VT}_{\psi}$:
\begin{itemize}
    \item[\textbf{(i)}] Sample imitation trajectories $\tau^{\VT}_i \sim (\widetilde{\pi}^{\VT}_{\psi}, \pi^{\AV})$.
    \item[\textbf{(ii)}] Update the discriminator $D_w(s,a^{\VT})$ with the gradients:
    \begin{align}
    \bbE_{\tau^{\VT}_i}\left[\sum_{t}\gamma^t \nabla_{w}\log (D_w(\cdot)) \right]
    + \bbE_{\tau^{\VT}_E}\left[\sum_{t}\gamma^t \nabla_{w}\log (1 - D_w(\cdot))  \right]\label{eq:update-Dw}
    \end{align}
    where $\tau^{\VT}_E$ are historical trajectories of the victim collected from interactions between the adversary and the victim. 
    \item[\textbf{(iii)}] Update $\psi$ with the gradients:
    \begin{multline}\label{eq:update-psi}
    \bbE_{\tau^{\VT}_i}\left[\sum_{t} \gamma^t \eta(s_t, a^{\AV}_t, a^{\VT}_t) \sum_t\nabla_{\psi} \log \pi^{\VT}_\psi(a^{\VT}_t|s_t) \right]
    - \lambda \nabla_{\psi} H(\widetilde{\pi}^{\VT}_{\psi}),
    \end{multline}
     which is a standard policy gradient update, for which one can use TRPO \citep{schulman2015trust} or PPO \citep{schulman2017proximal}.
\end{itemize}
For all the updates above, we interact with an adversary of policy $\pi^\AV$. This policy will keep changing during the adversarial policy learning and affect the victim's environmental dynamics. This would make the imitation learning process unstable and challenging to handle. We analyze the effect of the adversary's  policy on the imitating policy in Section~\ref{sec:Effect-Adv-policy-IM}.   

\subsection{Effects of the Adversary's Policy on the Victim Imitation Policy}
\label{sec:Effect-Adv-policy-IM}
It is important to see that our imitation learning differs from the standard GAIL as the dynamics $\bq^\VT(\pi^{\AV})$ are dependent of the adversary policy $\pi^\AV$ and our imitation learning model will be trained simultaneously with the adversary's policy. This raises  questions of how  the learning of the imitating policy is affected by such changing dynamics, and whether one can get a desired imitating policy when the adversary's policy gets stable. To answer these questions, let use consider the following victim's expected reward as a function of adversary's policy.  
\[
\Gamma(\pi^\AV) = \bbE_{\tau \sim \pi^\VT} \left[\sum_{t} \gamma^t r^\VT(s_t) \Big| \bq^\VT(\pi^{\AV})\right].
\]
That is, $\Gamma(\pi^\AV)$ is the expected reward that the victim can get by running a fixed policy $\pi^\VT$ when the adversary policy is $\pi^\AV$. Lemma~\ref{lm:lm2} establishes a bound for the gap $|\Gamma(\pi^\AV) - \Gamma(\widetilde{\pi}^\AV)|$, which implies that $\Gamma(\widetilde{\pi}^\AV)$ will converge to $\Gamma({\pi}^\AV)$ if $\widetilde{\pi}^\AV$ gets close to ${\pi}^\AV$.
\begin{lemma}
\label{lm:lm2}
Given two adversary policies $\pi^\AV$ and $\widetilde{\pi}^\AV$, let  $\cH = \max_{s} \left\{\left|V_{\pi^\VT}({s}|~\bq^\VT(\pi^{\AV}))\right|\right\} $. We obtain the following bound:
\[
\left|\Gamma(\widetilde{\pi}^{\AV})  - \Gamma(\pi^{\AV}) \right| \leq   \frac{\gamma \cH\sqrt{2\ln 2}}{1-\gamma}  \max_{s \in S} \left\{\sqrt{D_\KL(\pi^\AV (\cdot|s)||\widetilde{\pi}^\AV (\cdot|s))} \right\}
\]
where $D_\KL(\pi^\AV (\cdot|s)||\widetilde{\pi}^\AV (\cdot|s))$ is the KL divergence between the two  adversary policies $\widetilde{\pi}^\AV$ and ${\pi}^\AV$.
\end{lemma}
To prove the above  lemma, we extend the concept of the advantage function popularly used  in single-agent RL \citep{kakade2002approximately,schulman2015trust} to introduce the following  victim's \textit{competitive advantage function}, for any two states $s,\overline{s}$, conditional on adversary's policy $\pi^\AV$,
\[
A_{\pi^{\AV}}(s,\overline{s}) = r^\VT(s)+ \gamma V_{\pi^\VT}(\overline{s}|~\bq^\VT(\pi^{\AV})) - V_{\pi^\VT}({s}|~\bq^\VT(\pi^{\AV})).
\]
This allow use to write the expected reward $\Gamma(\pi^\AV)$ in terms of the expected long-term competitive advantage function over another adversary policy $\widetilde{\pi}^\AV$.
\[ 
\Gamma(\widetilde{\pi}^{\AV})  - \Gamma(\pi^{\AV})  =  \bbE_{\tau \sim \pi^{\VT}} \left[\sum_{t=0}\gamma^t\Big(A_{\pi^\AV} (s_t,s_{t+1})\Big)\Big|~ \bq^\VT(\widetilde{\pi}^{\AV})\right]
\]
with a note that $\bbE_{s_{t+1} \sim \bq(\pi^\AV)}[A_{\pi^\AV} (s_t,s_{t+1})] = 0$. 
This identity expresses the expected return of the adversary policy $\widetilde{\pi}^\AV$ over another policy $\pi^\AV$ in terms of victim's expected rewards. The competitive advantage function can be further bounded as $$\bbE_{\overline{s}\sim \pi^\VT, \bq^\VT(\widetilde{\pi}^\AV)}[A_{\pi^{\AV}}(s,\overline{s})] \leq \gamma \cH \max_{s}||\pi^\AV(\cdot|s) - \widetilde{\pi}^\AV(\cdot|s)||_1,$$ 
which can further bounded by $$\gamma \cH  \max_{s} \left\{\sqrt{2\ln 2D_\KL(\pi^\AV (\cdot|s)||\widetilde{\pi}^\AV (\cdot|s))} \right\}.$$ The full proof is given in the appendix. The proof of Lemma \ref{lm:lm2}  also reveals a bound based on maximum norm  $|\Gamma(\widetilde{\pi}^{\AV})  - \Gamma(\pi^{\AV})|\leq \frac{\cH\gamma}{1-\gamma}||\widetilde{\pi}^{\AV} - \widetilde{\pi}^{\AV}||_\infty$, implying that   $\Gamma(\pi^{\AV})$ is Lipschitz continuous in $\pi^{\AV}$ with Lipschitz constant $\frac{\cH\gamma}{1-\gamma}$.

Now, let $\pi^{\AV*}$ be a target adversary's policy that the imitator should be trained with. This would be a trained adversary's policy after the adversarial policy learning. If the imitating policy is trained with another adversary's policy, we aim to explore how this imitating policy performs under the target policy $\pi^{\AV*}$. Theorem \ref{th:th2} below establishes a performance guarantee for the imitating policy if it is trained with  a different adversary's policy. 
\begin{theorem}
\label{th:th1}
Suppose that  discriminator's network model $D$ of \eqref{prob:E-GAIL} varies within $[D^L,D^U] \subset [0,1]$. Let $\pi^{\AV*}$ be the target adversary policy that we want to train the imitation policy with, and let $(\widetilde{\pi}^{\VT*}, D^{\VT*}) $ be the imitation policy and the imitator's discriminator trained with another adversary $\pi^\AV$, we have the following performance guarantee for $\widetilde{\pi}^{\VT*}$.
\begin{align}
\Big| \phi^E(\widetilde{\pi}^{\VT*},D^{\VT*}|{\pi}^{\AV*})& - \max_{\widetilde{\pi}^{\VT}}\min_{D}\{\phi^E(\widetilde{\pi}^{\VT},D|{\pi}^{\AV*})\}\Big| \nonumber \\
& \leq 2K \max_{s \in S} \left\{\sqrt{D_\KL(\pi^\AV (\cdot|s)||{\pi}^{\AV*} (\cdot|s))} \right\},\nonumber
\end{align}
where 
\[
K = \frac{\gamma \sqrt{2\ln 2} \left(\max_{s}\{r^\VT(s)\} - \log (D^L-D^L D^U)\right)}{(1-\gamma)^2}.
\]
\end{theorem}
Since the adversary policy $\pi^\AV$ will keep changing during our adversarial policy optimization,  Theorem \ref{th:th1} implies that the imitating policy will be stable if $\pi^\AV$ becomes stable, and if $\pi^\AV$ is approaching the target adversary's policy, the imitator's policy also converges to the one that is trained with the target adversary policy with rate $\cO\left(\sqrt{D_\KL(\pi^\AV||\pi^{\AV*})}\right)$. In other words, if the actual policy $\pi^\AV$ is not not too far from the target $\pi^{\AV*}$ such that $D_\KL(\pi^\AV||\pi^{\AV*})\leq \epsilon$, then the expected return of the imitating policy is within a $\cO(\sqrt{\epsilon})$ neighbourhood of the desired ``expected return''.

\section{Adversarial Policy Training}
We now discuss our main adversarial policy learning algorithm. We start by explaining the adversarial policy model introduced in~\cite{Guo2021AdversarialPL}, upon which we build our new adversarial policy learning algorithm. We then introduce our integration of the victim imitator and our new corresponding main learning algorithm. Finally, we provide our theoretical analysis on the worst-case performance of our learning algorithm when the victim's policy is not fixed.  
\subsection{Adversarial Policy Learning with Integration of Victim Imitator}
Similar  to prior works, we assume that the policy of victim is fixed and the  aim is to learn an adversary policy to maximize the chances of winning (or win and tie). Similarly to \cite{Guo2021AdversarialPL}, we train the  policy by maximizing the following enhanced objective  
\begin{equation}\label{prob:paper2}
\max_{\pi^{\AV}}\left\{ V_{\pi^{\AV}}(s_0) - V_{\pi^{\VT}} (s_0|\bq^\VT(\pi^{\AV})) \right\} 
\end{equation}
where $V_{\pi^{\AV}} (s_0)$ is the expected long-term reward of the adversary by following the policy $\pi^{\AV}$ but the transition probabilities are affected by the victim policy $\pi^{\VT}$. The objective  in \eqref{prob:paper2} involves both the value functions of the adversary and victim, in which the value function of the victim depends on the adversary's policy though the environment dynamics. This complication makes \eqref{prob:paper2} not straightforward to solve.
In Proposition \ref{prop:EAIL} below we show how to compute the policy gradient of the enhanced adversarial training model in \eqref{prob:paper2}, based on which we can show the RL problem in \eqref{prob:paper2} can be converted into a standard competitive game.

\begin{proposition}\label{prop:EAIL}
The gradient of \eqref{prob:paper2} w.r.t  adversary's policy can be computed as follows:
\begin{multline*}
\nabla_{\theta} \left( V_{\pi^{\AV}_\theta}(s_0) - V_{\pi^{\VT}}(s_0|\bq^\VT(\pi^{\AV}_\theta))\right) \\
= \bbE_{\tau\sim (\pi^{\VT}, \pi^{\AV})} \left[  \Delta^R(\tau)\sum_{t} \nabla_{\theta} \log \pi^{\AV}_\theta(a^{\AV}_t|s_t) \right].
\end{multline*}
where $\Delta^R(\tau) =\sum_{t}\gamma^t \left[ r^{\AV}(s_t)- r^{\VT}(s_t)\right] $
\end{proposition}
Similarly to Corollary \ref{coro:c1}, the RL problem in \eqref{prob:paper2} can be converted into a standard competitive game with differentiated rewards $r^{\AV}(s_t)- r^{\VT}(s_t)$ and fixed victim's policy $\pi^{\VT}$. Thus, the environmental dynamics are fixed and a standard RL algorithm can apply.  
\begin{corollary}
\label{coro:c2}
\eqref{prob:paper2} is equivalent to
\begin{equation}
\label{prob:AIPL}
 \max_{\pi^{\AV}} \left\{ \bE_{\tau \sim \pi^{\AV}}\left[\sum_{t=0}^\infty \gamma^t \Delta^r(s_t)\;\Big| \bq^\AV(\pi^{\VT}) \right] \right\}
\end{equation}
where $\Delta^r(s_t) =  r^{\AV}(s_t)- r^{\VT}(s_t)$.
\end{corollary}

\paragraph{\textbf{Integrating the victim imitator}}
To integrate the imitation learning model to the adversarial policy optimization, we use the imitating policy $\widetilde{\pi}^{\VT}$ to predict  victim's intention (i.e., next actions)  and include this information into the state space of the adversary. Intuitively, we support the adversary by providing it more information about the  victim's next moves. Based on Corollary~\ref{coro:c2}, we train the adversary's policy by solving the optimization problem \eqref{prob:AIPL}
where the adversary's policy is now of the form $\pi^\AV \Big(a_t^{\AV}|s_t
, \widetilde{a}^{\VT}_t   \Big)$. 
That is, the  adversary's policy now is conditional on current state $s_t$ as well as a predicted next victim action $\widetilde{a}^{\VT}_t$ provided by the imitating policy model $\widetilde{\pi}^{\VT}$, and the  imitating policy model $\widetilde{\pi}^{\VT}$ takes the state $s_t$ to predict the next victim's actions. Finally, \eqref{prob:AIPL} can be solved using standard policy gradient algorithms such as PPO.


\paragraph{\textbf{Main learning algorithm}}
Putting all the results developed above together, we present our adversarial policy learning in Algorithm \ref{algo:algo1}.  Figure \ref{fig:train-attacker} illustrates the three components of our algorithm, including the adversary, the victim and the imitator, and connections between these three components. In short, both the adversary policy and imitator policy are simultaneously updated during interactions between the adversary agent and the victim agent. Observed trajectories are transferred to the imitator to update the imitating policy, following steps in Section \ref{sec:IM}. Simultaneously, our algorithm provides the imitator with the victim's current state to ask for victim's intention. This information will be passed to the adversary's policy network to update the policy adversarial learning. It is expected that when the adversary's policy gets stable and demonstrations from the victim agent are sufficient, the imitation policy also gets close to a desired one (as shown in Section \ref{sec:Effect-Adv-policy-IM}) and the imitator is able to accurately predict victim agent's next actions. Algorithm \ref{algo:algo1} shows the main steps of our adversarial policy training algorithm and we give a more detailed version in the appendix.   

\begin{algorithm}[t!]
\caption{Adversarial Policy Imitation Learning (brief version)}
\label{algo:algo1}
\begin{algorithmic}
\State \textbf{Input:} Adversary's policy network $\pi^\AV_\theta$; imitator's policy network $ \widetilde{\pi}^\VT_\psi$; imitator's discriminator $D_w$; initial parameters $\theta_0, \psi_0, w_0$.
\For{$i=0,1,2,...$}
   \State \comments{Updating imitator's policy}
  \State Sample trajectories $\tau_i \sim \pi^\AV_{\theta_i}, \pi^\VT$.
  \State Update discriminator $D_w$ network from $w_i$ to $w_{i+1}$ using \eqref{eq:update-Dw}.
  \State Update imitator's policy network from $\psi_i$ to $\psi_{i+1}$ based on TRPO or PPO using \eqref{eq:update-psi}.
   \State \comments{Updating adversary's policy}
  \State Generate imitator's predicted actions $\widetilde{a}^\VT \sim \pi^\AV_{\psi_{i+1}}$
  \State Update the adversary policy  from $\theta_i$ to $\theta_{i+1}$ based on TRPO or PPO, using the gradients in \eqref{prop:EAIL}.
\EndFor
\end{algorithmic}
\end{algorithm}

\begin{figure}
\caption{An overview of our adversary policy training algorithm. At each step $t$, state $s_t$ is forwarded to victim policy $\pi^\VT$ and to the imitator's generator policy $\widetilde{\pi}^\VT_{\psi}$ to generate the victim action $a^\VT_t$ and the imitator action $\widetilde{a}^\VT_t$ respectively. The adversary policy $\pi^\AV_\theta$ uses this imitator output $\widetilde{a}^\VT_t$ and current state $s_t$ to generate the adversary action $a^\AV_t$. The environment then transits to the next state given the combination of actions $(a^\VT_t, a^\AV_t)$. For each step, victim and imitator transitions are appended to corresponding trajectory buffers, which will be used to update the discriminator of the imitator $D_w$.} 
\label{fig:train-attacker}
\centering
\vspace{0.2cm}
\includegraphics[width=0.45\textwidth]{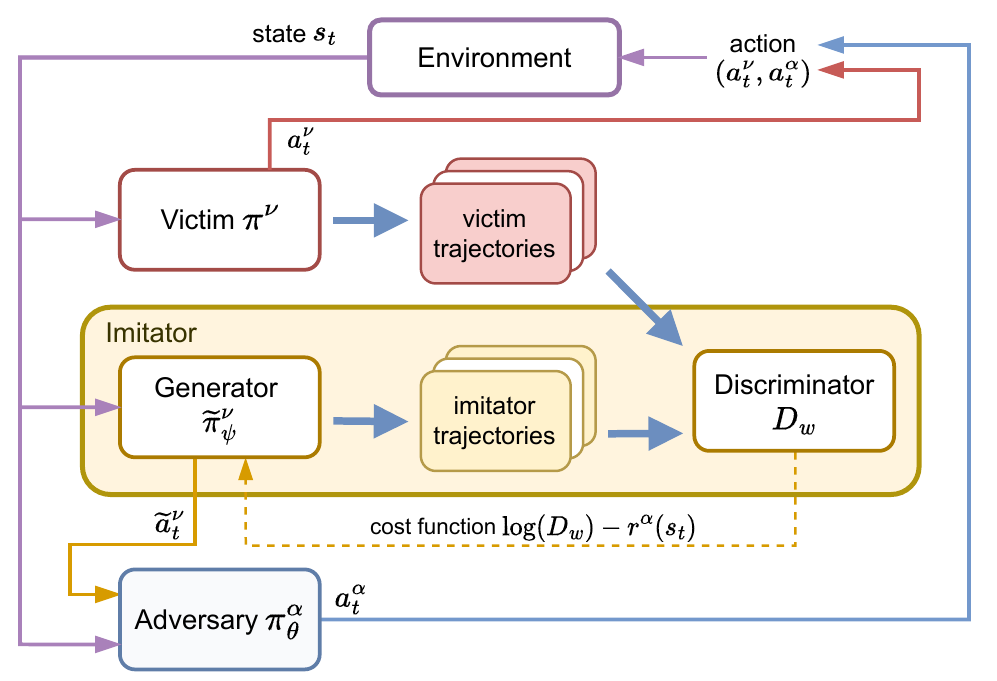}
\end{figure}

\subsection{Worst-case Performance When Training the Adversary with Unfixed Victim's Policy}
So far we train the adversary policy by assuming that the victim agent always follows a fixed policy. We explore, in this section, the question that,  if the victim's policy is not fixed, how the victim's unstable policy  would affect the adversarial  policy learning. To start our analysis, let $\pi^{\VT}_0$ be a ``true'' victim policy that the adversary agent should be trained with and suppose that, due to external causes, the adversary agent is only trained with victim policies that vary within the following set: 
$$\Omega(\epsilon) =  \{\pi^\VT\Big|~\max_{s\in \cS} D_{\KL}(\pi^\VT(s) || \pi^\VT_0 (s)) \leq \epsilon\}.$$ 
We define the worst-case expected return of the adversary agent when being trained with such varying victim policies 
$
 Y(\epsilon) =  \min_{\pi^\VT \in \Omega(\epsilon)}~\max_{\pi^{\AV}} \left\{ \bE_{\tau \sim (\pi^{\AV})}\left[\sum_{t=0}^\infty \gamma^t \Delta^r(s_t)\;\Big| \bq(\pi^{\VT}) \right] \right\}.
$
The following theorem gives a bound for the gap between the worst-case and the desired expected  return obtained from training with the "true" victim policy $Y^* = \max_{\pi^{\AV}} \Big\{ \bE_{\tau}\Big[\sum_{t} \gamma^t \Delta^r(s_t)\;\Big| \bq^\AV(\pi^{\VT}_0) \Big] \Big\}$. 
\begin{theorem}\label{th:th2}
For any $\epsilon>0$, we have the following bound
\[
\left|Y(\epsilon) - Y^* \right| \leq  \frac{\gamma\sqrt{2\ln 2} \max_s\{ |\Delta^r(s)|\}}{(1-\gamma)^2} \sqrt{\epsilon}. 
\]
\end{theorem}
The above bound implies that the worst-case performance of the adversarial training would not too bad (i.e., within a neighbourhood $\cO(\sqrt{\epsilon})$) if the victim policy  that the adversary is trained with is not too far from the ``true'' victim policy. On the other hand, if the adversary is trained with an arbitrary victim policy, the training outcomes would be very bad. Let us use the Rock-paper-scissors game to illustrate this. If  the victim always plays ``rock'' during the adversary's training, it will not take long for the adversary agent to see that playing ``paper'' always gives a 100\% winning rate. But if the victim change their policy to playing ``scissors'', then that trained adversary's policy will always yield a 0\% winning rate.   

\section{Evaluation}
We evaluate our proposed Enhanced Adversarial Policy Imitation Learning (\textbf{E-APIL}) algorithm, i.e., the adversarial policy learning \eqref{prob:AIPL} with an enhanced imitator \eqref{prob:E-GAIL}, and the non-enhanced imitation version of our algorithm (named, \textbf{APIL}), i.e., the adversarial policy learning \eqref{prob:AIPL} with a non-enhanced imitator \eqref{prob:AIL}. 
 As to do so, we use four competitive MuJoCo game environments introduced by Emergent Complexity (EC) \cite{Bansal2018EmergentCV}, including \emph{Kick And Defend,} \emph{You Shall Not Pass}, \emph{Sumo Humans}, and \emph{Sumo Ants}. We compare the performance of our algorithms with: (i) the well-trained adversary/victim agents in \cite{Bansal2018EmergentCV} which we consider as \textbf{Baseline} agents; (ii) Attacking Deep Reinforcement Learning (\textbf{ADRL}) \cite{Gleave2020AdversarialPA}; and (iii)  Adversarial Policy Learning (\textbf{APL}) \cite{Guo2021AdversarialPL}. The two methods, ADRL and APL, are the state-of-the-art methods in adversarial policy learning. 
For fair comparisons, we use the same experiment settings (i.e., pre-trained parameters, hyperparameters, and evaluation metrics) as in \cite{Gleave2020AdversarialPA,Guo2021AdversarialPL}. Implementation details are specified in supplementary section.

\subsection{Adversarial Policy Performance: Training Adversary against Baseline Victim}
In this experiment, we train our adversary agent using our proposed algorithms to play against the baseline victim agent~\cite{Bansal2018EmergentCV}. We aim to examine if our generated adversarial policy can trigger the victim agent to perform poorly. 
Table \ref{tab:adv-performance} shows the winning rate (i.e., numbers in white cells) and the winning plus tie rate (numbers in gray cells) of our trained adversary agents playing against the  \textit{baseline} victim agent,  compared to those trained by other adversarial policy algorithms (also against the \textit{baseline} victim). 
Each reported value is calculated based on 1000 different rounds of game playing. 

\begin{figure}[t!]
\centering
\caption{Illustrative snapshots of a victim (in blue) against normal and adversarial opponents (in red) in SumoHumans simulator. Two players of the baseline method try to get close to each other and butt their opponents to win. However, APL learns to kneel to stay in the ring and its victims may find it harder to knock it down. Our algorithm even learns to stand better with two knees and dodge attacks from the victim.}
 \begin{tabular}{@{}cc@{}@{}c@{}@{}c@{}@{}c@{}} 
\begin{tabular}[c]{@{}c@{}} \rotatebox{90}{\textbf{Baseline}} \end{tabular} &
\showtsnapshot[0.1]{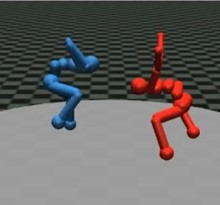} &
\showtsnapshot[0.1]{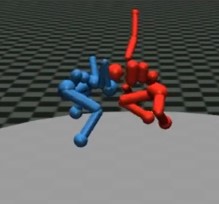} &
\showtsnapshot[0.1]{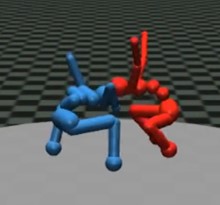} &
\showtsnapshot[0.1]{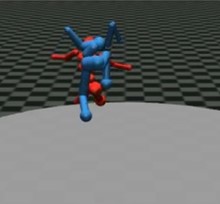}
\vspace{0.2cm} \\ 
\begin{tabular}[c]{@{}c@{}} \rotatebox{90}{\textbf{APL}} \end{tabular} &
\showtsnapshot[0.1]{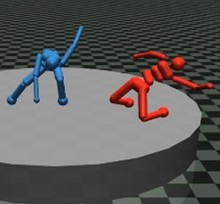} &
\showtsnapshot[0.1]{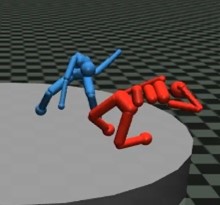} &
\showtsnapshot[0.1]{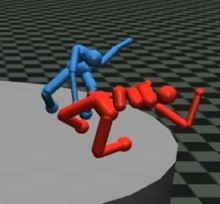} &
\showtsnapshot[0.1]{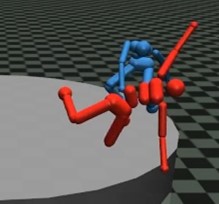}
\vspace{0.2cm} \\
\begin{tabular}[c]{@{}c@{}}\rotatebox{90}{\textbf{E-APIL}}\end{tabular} &
\showtsnapshot[0.1]{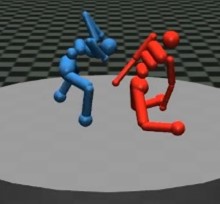} &
\showtsnapshot[0.1]{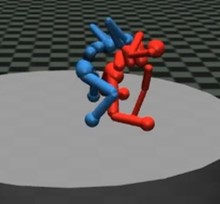} &
\showtsnapshot[0.1]{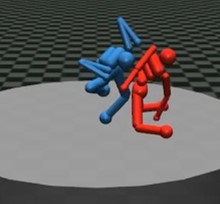} &
\showtsnapshot[0.1]{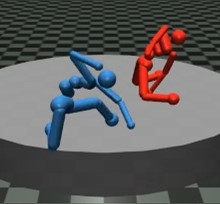}
 \end{tabular}
\end{figure}

Overall, our APIL and E-APIL methods achieve 
significantly higher winning rates for Kick-And-Defend and Sumo-Humans compared to all  existing methods. In You-Shall-Not-Pass, our method E-APIL achieves a winning rate which is only 1\% less than the best-performed method (APL) in this game environment, while significantly outperforming the others. On the other hand, in Sumo-Ants, we observe an interesting phenomenon. While we obtain the best winning-plus-tie rates in Sumo-Ants, we obtain a lower winning rate compared to the baseline adversary. This phenomenon also holds true for existing algorithms (ADRL and APL). The cause of this phenomenon comes from a unique underlying characteristic of Sumo-Ants, i.e.,  it is very challenging to reach the \emph{win} outcome --- the victim has a high chance to reach a draw outcome  by  just jumping to the ground without touching opponent.
As a result, our adversary was essentially trained to optimize the policy towards \emph{draw} outcomes in Sumo-Ants, at the sacrifice of the win rate. 

\begin{table}[t!]
\centering
\caption{Winning rate (white) and winning plus tie rate (gray) of new adversary vs baseline victim.}
\label{tab:adv-performance}
\begin{tabular}{c|g|g|g|g|g}
\rowcolor{white}
\textbf{} &
  \textbf{\begin{tabular}[c]{@{}c@{}}Base-\\ line\end{tabular}} &
  \textbf{ADRL} &
  \textbf{APL} &
  \begin{tabular}[c]{@{}c@{}}\textbf{APIL}\\ (ours)\end{tabular} &
  \begin{tabular}[c]{@{}c@{}}\textbf{E-APIL}\\ (ours)\end{tabular} \\ \hline
\rowcolor{white}
\multirow{2}{*}{\begin{tabular}[c]{@{}c@{}}Kick And\\ Defend\end{tabular}}   & 28\% & 48\% & 80\% & 85\% & \red{89\%} \\
                                                                             & 29\% & 49\% & 80\% & 86\% & \red{90\%} \\ \hline
\rowcolor{white}
\multirow{2}{*}{\begin{tabular}[c]{@{}c@{}}You Shall\\ Not Pass\end{tabular}} & 38\% & 56\% & \red{68\%} & 58\% & 67\% \\
                                                                             & 38\% & 56\% & \red{68\%} & 58\% & 67\% \\ \hline
\rowcolor{white}
\multirow{2}{*}{\begin{tabular}[c]{@{}c@{}}Sumo\\ Humans\end{tabular}}       & 7\%  & 22\% & 36\% & \red{73\%} & 72\% \\
                                                                             & 7\%  & 61\% & 77\% & \red{88\%} & 87\% \\ \hline
\rowcolor{white}
\multirow{2}{*}{\begin{tabular}[c]{@{}c@{}}Sumo\\ Ants\end{tabular}}         & \red{39\%} & 10\% & 2\%  & 2\%  & 3\%  \\
                                                                             & 56\% & 41\% & \red{81\%} & \red{81\%} & 80\%
\end{tabular}
\end{table}


\begin{figure}[t!]
\caption{t-SNE visualizations of the victim activations when playing against different opponents in MuJoCo games.}\label{fig:t-SNE}
\centering
\begin{subfigure}{0.45\textwidth}
  \centering
  \includegraphics[width=1.0\linewidth, trim={0.8cm 6.5cm 0 0},clip]{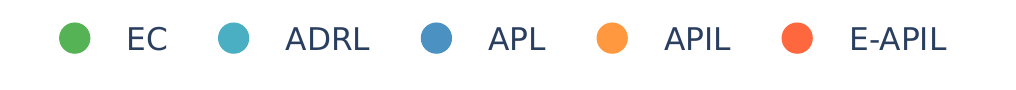}
  \captionsetup{justification=centering}
\end{subfigure}
\\
\showtsne[0.23]{KickAndDefend-v0}{tsne}{Kick And Defend}
\showtsne[0.23]{YouShallNotPassHumans-v0}{tsne}{You Shall Not Pass}
\\
\showtsne[0.23]{SumoHumans-v0}{tsne}{Sumo Humans}
\showtsne[0.23]{SumoAnts-v0}{tsne}{Sumo Ants}
\end{figure}

We now seek to better understand why our methods get higher winning rates than other algorithms. 
Figure \ref{fig:t-SNE} shows t-SNE visualization \citep{van2008visualizing} of the trained adversary against the baseline victim  by recording victim's 
policy activations. The t-SNE visualizations for all four game environments indicates that our algorithms APIL/E-APIL seek to activate different policy distribution regions of the victim (the orange and red regions) compared to existing algorithms, allowing our policy learning to converge to a better optimum.

Table \ref{tab:adv-performance} also shows that ADRL and APL are more focused on getting draw in the Sumo games than learning how to win. We plot in
Figure \ref{fig:adv-performance} the training performance of our algorithms for the four game environments, which
show that the tie rates are already high during early episodes. As mentioned, the victim in these games can easily get a draw by just jumping to the ground without touching the opponent, which makes the tie rates very high.  


Finally, Table \ref{tab:adv-performance} shows that our E-APIL with an enhanced imitator is significant better than APIL in both Kick-and-Defend and You-Shall-Not-Pass. This result implies that incorporating  the  adversary's expected rewards  into the imitator's value function 
definitely helps improve the quality of the  generated adversarial policy. 
In Sumo-Humans and Sumo-Ants where the tie rates account for a  large proportion of the outcomes, the performance of E-APIL and APIL are not substantially disparate. 

\begin{figure*}[htb]
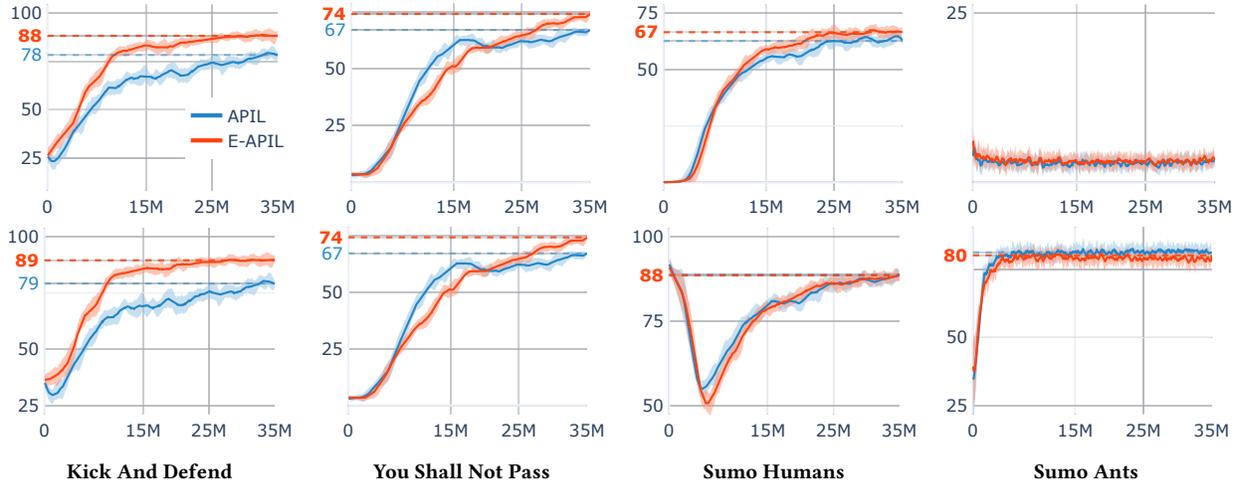

\centering
\caption{Performance of newly trained adversary vs. baseline victim while training against baseline victim. First row: win-rate. Second row: win-rate + tie-rate. Blue curves: AIPL, red curves: E-AIPL. In \textit{You-Shall-Not-Pass-Humans}, the tie rates are always zero because there is no declaration for a tie game.
} 
\label{fig:adv-performance}
\showinstance[0.22]{KickAndDefend-v0}{adv_win_legend}{}
\showinstance[0.22]{YouShallNotPassHumans-v0}{adv_win}{}
\showinstance[0.22]{SumoHumans-v0}{adv_win}{}
\showinstance[0.22]{SumoAnts-v0}{adv_win}{}
\showinstance[0.22]{KickAndDefend-v0}{adv_win+tie}{Kick And Defend}
\showinstance[0.22]{YouShallNotPassHumans-v0}{adv_win+tie}{You Shall Not Pass}
\showinstance[0.22]{SumoHumans-v0}{adv_win+tie}{Sumo Humans}
\showinstance[0.22]{SumoAnts-v0}{adv_win+tie}{Sumo Ants}
\end{figure*}

\subsection{Blinding the Trained Adversary}
To further understand the role of the imitator behind the efficiency of our adversarial training algorithms,  
we conduct the following experiment. First, we take the trained adversary agents  and let them play with the baseline victim, but now we 
blind the adversary's observation on the victim or, in other words, zero out the adversary observation pertaining to the victim. 
By blinding the adversary agents, we aim to demonstrate that the trained adversary still manages to make the victim to perform poorly just based on the imitator's policy output, despite of the \emph{blinded} disadvantage.

Tables 
\ref{tab:adv-performance-blind-adv} reports our experiments with blinded trained  adversary against the baseline victim. In general, 
 our APIL and E-APIL adversary agents outperform those trained by the other methods when playing against the baseline victim.
 Intuitively, even when being blinded, by taking feedback from the trained imitator, our trained adversary agents would still partially predict victim's intention to make better actions, compared to those trained by other methods.

\begin{table}[t!]
\centering
\caption{Winning rate (white) and winning plus tie rate (gray) of  games between  blinded adversary and baseline victim 
}
\label{tab:adv-performance-blind-adv}
\begin{tabular}{c|g|g|g|g|g}
\rowcolor{white}
\textbf{} &
  \textbf{\begin{tabular}[c]{@{}c@{}}Base-\\ line\end{tabular}} &
  \textbf{ADRL} &
  \textbf{APL} &
  \begin{tabular}[c]{@{}c@{}}\textbf{APIL}\\ (ours)\end{tabular} &
  \begin{tabular}[c]{@{}c@{}}\textbf{E-APIL}\\ (ours)\end{tabular} \\ \hline
\rowcolor{white}
\multirow{2}{*}{\begin{tabular}[c]{@{}c@{}}Kick And\\ Defend\end{tabular}}   & 28\% & 48\% & 80\% & 85\% & \red{89\%} \\
                                                                             & 29\% & 49\% & 80\% & 86\% & \red{90\%} \\ \hline
\rowcolor{white}
\multirow{2}{*}{\begin{tabular}[c]{@{}c@{}}You Shall\\ Not Pass\end{tabular}} & 1\%  & 48\% & 65\% & 62\% & \red{68\%} \\
                                                                             & 1\%  & 48\% & 65\% & 62\% & \red{68\%} \\ \hline
\rowcolor{white}
\multirow{2}{*}{\begin{tabular}[c]{@{}c@{}}Sumo\\ Humans\end{tabular}}       & 3\%  & 0\%  & 1\%  & \red{69\%} & 66\% \\
                                                                             & 7\%  & \red{83\%} & 62\% & \red{83\%} & 80\% \\ \hline
\rowcolor{white}
\multirow{2}{*}{\begin{tabular}[c]{@{}c@{}}Sumo\\ Ants\end{tabular}}         & \red{13\%} & 7\%  & 3\%  & 2\%  & 3\%  \\
                                                                             & 42\% & 38\% & 64\% & \red{81\%} & 79\%
\end{tabular}
\end{table}

In summary, our methods works well in interactive environments, even in the blinding setting, thanks to the capability of predicting opponent's intention through the trained imitator.

\subsection{Improving Victim Resiliency: Retraining Victim against  New Adversary}
Previous studies demonstrate that one could retrain the victim and thus improve its adversary resistance \cite{Gleave2020AdversarialPA, Guo2021AdversarialPL}. In this experiment, we also retrain the victim agent against the newly trained adversary agent to examine the resistance of the retrained victim agent against adversarial policies. We further explore the resilience transferability of the retrained victim agents. Specifically, similar to previous work, we retrain the victim agent against a mixed adversary agent of the new adversary (whose policy is trained based on one of the evaluated adversarial training algorithms (i.e., baseline, ADRL, APL and ours) and the baseline adversary.
We then have the retrained victim agent play against the baseline adversary for 1000 rounds and report its winning as well as winning plus tie rates. 


Table \ref{tab:vic-ori-performance} shows the winning and winning plus tie rates of the retrained victim agents playing against the baseline adversary. For example, the \textbf{E-APIL} column shows the game results between the victim agent (retrained based on interactions with an E-APIL adversary) and the Baseline adversary agent. Except for the Sumo-Ants game, our methods outperform other algorithms in terms of retraining the victim agent to be stronger. 
For all the games, the baseline victim generally yields good results, which is not a surprising observation as the OpenAI's baseline agents \citep{Bansal2018EmergentCV} are well trained against their opponents with around 1B steps and 4 GPUs. In our experiments, 
with only 35M + 10M steps and 1 GPU, we are able to make the  victim agent significantly stronger. Moreover, despite the fact that the APIL/E-APIL based victim agents are retrained against our APIL/E-APIL adversary agents, these retrained victim agents still manage to perform well against the baseline adversary as show in the last two columns of Table \ref{tab:vic-ori-performance}. This result clearly shows that the strong resilience of APIL/E-APIL based victim agents can be transferred to other game settings with different types of adversary agents (e.g., the baseline adversary in this experiment). 

%

\begin{table}[htb]
\centering
\caption{Winning rate (white) and winning plus tie rate (gray) of retrained victim agents vs baseline adversary.}
\label{tab:vic-ori-performance}
\begin{tabular}{c|g|g|g|g|g}
\rowcolor{white}
\textbf{} &
  \textbf{\begin{tabular}[c]{@{}c@{}}Base-\\ line\end{tabular}} &
  \textbf{ADRL} &
  \textbf{APL} &
  \begin{tabular}[c]{@{}c@{}}\textbf{APIL}\\ (ours)\end{tabular} &
  \begin{tabular}[c]{@{}c@{}}\textbf{E-APIL}\\ (ours)\end{tabular} \\ \hline
\rowcolor{white}
\multirow{2}{*}{\begin{tabular}[c]{@{}c@{}}Kick And\\ Defend\end{tabular}}   & 71\% & 62\% & 70\% & \red{87\%} & \red{87\%} \\
                                                                             & 72\% & 70\% & 77\% & 89\% & \red{90\%} \\ \hline
\rowcolor{white}
\multirow{2}{*}{\begin{tabular}[c]{@{}c@{}}You Shall\\ Not Pass\end{tabular}} & 62\% & 64\% & 63\% & 71\% & \red{72\%} \\
                                                                             & 62\% & 64\% & 63\% & 71\% & \red{72\%} \\ \hline
\rowcolor{white}
\multirow{2}{*}{\begin{tabular}[c]{@{}c@{}}Sumo\\ Humans\end{tabular}}       & 93\% & 76\% & 79\% & 94\% & \red{95\%} \\
                                                                             & 93\% & 85\% & 84\% & 95\% & \red{96\%} \\ \hline
\rowcolor{white}
\multirow{2}{*}{\begin{tabular}[c]{@{}c@{}}Sumo\\ Ants\end{tabular}}         & \red{44\%} & 24\% & 30\% & 29\% & 33\%  \\
                                                                             & \red{61\%} & 38\% & 48\% & 52\% & 55\%
\end{tabular}
\end{table}

We further test the performance of each retrained victim agent against our E-APIL adversary and report the  winning and winning plus tie rates in Table \ref{tab:our-best-adv-vs-others}. For the two non-sumo games, the winning rates of ADRL/APL retrained victims are less than 60\%, which are significantly smaller than our rates.  Obviously, our E-APIL retrained victim achieves better results because it's trained against our E-APIL adversary, but our APIL also gets much better winning rates than ADRL/APL methods. It generally indicates the robustness and efficiency of our algorithms, compared to other approaches,  in terms of retraining the victim agent to have better versions of it.

\begin{table}[htb]
\centering
\caption{Winning rate (white) and winning plus tie rate (gray) of retrained victim agents vs our E-APIL adversary}
\label{tab:our-best-adv-vs-others}

\begin{tabular}{c|g|g|g|g|g}
\rowcolor{white}
\textbf{} &
  \textbf{\begin{tabular}[c]{@{}c@{}}Base-\\ line\end{tabular}} &
  \textbf{ADRL} &
  \textbf{APL} &
  \begin{tabular}[c]{@{}c@{}}\textbf{APIL}\\ (ours)\end{tabular} &
  \begin{tabular}[c]{@{}c@{}}\textbf{E-APIL}\\ (ours)\end{tabular} \\ \hline
\rowcolor{white}
\multirow{2}{*}{\begin{tabular}[c]{@{}c@{}}Kick And\\ Defend\end{tabular}}   & 10\% & 31\% & 52\% & 82\% & \red{91\%} \\
                                                                             & 11\% & 34\% & 53\% & 83\% & \red{92\%} \\ \hline
\rowcolor{white}
\multirow{2}{*}{\begin{tabular}[c]{@{}c@{}}You Shall\\ Not Pass\end{tabular}} & 33\% & 46\% & 58\% & 78\% & \red{88\%} \\
                                                                             & 33\% & 46\% & 58\% & 78\% & \red{88\%} \\ \hline
\rowcolor{white}
\multirow{2}{*}{\begin{tabular}[c]{@{}c@{}}Sumo\\ Humans\end{tabular}}       & 14\% & 32\% & 46\% & \red{51\%} & 41\% \\
                                                                             & 28\% & 84\% & 85\% & 84\% & \red{91\%} \\ \hline
\rowcolor{white}
\multirow{2}{*}{\begin{tabular}[c]{@{}c@{}}Sumo\\ Ants\end{tabular}}         & \red{20\%} & 17\% & 19\% & \red{20\%} & \red{20\%}  \\
                                                                             & \red{97\%} & 96\% & \red{97\%} & \red{97\%} & \red{97\%}
\end{tabular}

\end{table}

\section{Conclusion}
This paper introduces a new effective adversarial policy learning algorithm based on a novel integration of a new victim-imitation learning into the adversarial policy training process. Our victim-imitation component (which is an enhanced version of the state-of-the-art imitation method GAIL) discovers underlying characteristics of the victim agent, enabling the prediction of the victim's next moves which can be leveraged to strengthen the adversarial policy generation. We present important theoretical results on the inter-dependency between the victim-imitation learning and the adversarial policy learning, showing the convergence of our learning algorithm. We demonstrate the superiority of our proposed algorithm compared to existing adversarial policy learning algorithms through extensive experiments on various game environments.  





\bibliographystyle{ACM-Reference-Format} 
\bibliography{sample}


\clearpage
\appendix

\onecolumn
\begin{center}
{\huge \textbf{Appendix}}
\end{center}

\section{Missing Proofs}
\subsection{Proof of Lemma \ref{lm:lm1}}
\begin{lemma}
The gradient of $V_{\pi^{\AV}}(s_0|\bq^\AV(\widetilde{\pi}^{\VT}_\psi))$ w.r.t  $\psi$ can be computed as follows:
\begin{equation}
\nabla_\psi \left(V_{\pi^{\AV}}(s_0|\bq^\AV(\widetilde{\pi}^{\VT}_\psi))\right) = \bbE_{\tau\sim \widetilde{\pi}^{\VT}_\psi} \left[ R^{\AV}(\tau)\sum_{t} \nabla_{\psi}\log \widetilde{\pi}^{\VT}_\psi(a^{\VT}_t|s_t)\Big| \bq^\VT(\pi^\AV) \right]. 
\end{equation}
\end{lemma}
\begin{proof}
We write the adversary's expected reward as
\begin{align}
\bbE_{\tau\sim \pi^\AV} \left[ \sum_{t} \gamma^t r^\AV(s_t) \Big| \bq^\AV(\widetilde{\pi}^\VT_\psi) \right] &= \sum_{\tau} R^\AV(\tau)\prod_t \pi^\AV(a^\AV_t|s_t) P(s_{t+1}|s_t,a^\AV)\nonumber \\
&=  \sum_{\tau = \{(s_t,a^\AV_t, a^\VT_t)\}} R^\AV(\tau)\prod_t \pi^\AV(a^\AV_t|s_t) \widetilde{\pi}^\VT_\psi(a^\VT_t|s_t)P(s_{t+1}|s_t,a^\AV_t,a^\VT_t)
\end{align}
Taking the derivative of  the above expected value w.r.t. $\psi$ we get
\begin{align}
\nabla_\psi\left(\bbE_{\tau\sim \pi^\AV} \left[ \sum_{t} \gamma^t r^\AV(s_t) \Big| \bq^\AV(\widetilde{\pi}^\VT_\psi) \right]\right) 
&=  \sum_{\tau = \{(s_t,a^\AV_t, a^\VT_t)\}} R^\AV(\tau)P(\tau)\nabla_{\psi}\log \left(\prod_t \pi^\AV(a^\AV_t|s_t) \widetilde{\pi}^\VT_\psi(a^\VT_t|s_t)P(s_{t+1}|s_t,a^\AV_t,a^\VT_t)\right)\nonumber\\
&= \sum_{\tau = \{(s_t,a^\AV_t, a^\VT_t)\}} R^\AV(\tau)P(\tau) \sum_t \nabla_\psi\log\widetilde{\pi}^\VT_\psi(a^\VT_t|s_t) \nonumber \\
&= \bbE_{\tau \sim \widetilde{\pi}^\VT_\psi} \left[R^\AV(\tau) \sum_t \nabla_\psi\log\widetilde{\pi}^\VT_\psi(a^\VT_t|s_t)\Big|~\bq^\VT(\pi^\AV) \right],\nonumber
\end{align}
which is the desired equality. 
\end{proof}


\subsection{Proof of Proposition \ref{prop:IM-gradient}}

\begin{proposition}
The gradient of the objective \eqref{prob:E-GAIL} w.r.t. $\psi$ can be computed as follows: 
\[
\bbE_{\tau\sim \widetilde{\pi}^{\VT}_{\psi}}\left[\sum_{t} \gamma^t \eta(s_t, a^{\AV}_t, a^{\VT}_t) \sum_t \nabla_{\psi} \log \widetilde{\pi}^{\VT}_\psi(a^{\VT}_t|s_t) \right] 
- \lambda \nabla_{\psi} H(\widetilde{\pi}^{\VT}_{\psi})
\]
where 
$
\eta(s_t, a^{\AV}_t, a^{\VT}_t) = \log (D(s_t, a^{\VT}_t)) - r^{\AV}(s_t).
$
\end{proposition}
\begin{proof}
The first part of \eqref{prob:E-GAIL} is a standard long-term reward whose gradients can be computed as
\[
 \nabla_\psi \left(\phi( \widetilde{\pi}^\VT_\psi,D)\right) = \bbE_{\tau\sim \widetilde{\pi}^\VT_\psi}\left[ \sum_{t}\gamma^t \log (D(s_t,a^\AV_t))\sum_t \log \widetilde{\pi}^\VT_{\psi}(a^\VT_t|s_t)\Big|~ \bq^\VT(\pi^\AV)\right]\nonumber
\]
Combine this with the derivation in Lemma \ref{lm:lm1} we get
\[
\nabla_\psi \left(\phi^E( \widetilde{\pi}_\psi,D)\right) = \bbE_{\tau\sim \widetilde{\pi}_\psi}\left[ \sum_{t}\gamma^t \Big(\log (D(s_t,a^\AV_t)) - r^\AV(s_t)\Big)\sum_t \log \widetilde{\pi}^\VT_{\psi}(a^\VT_t|s_t)\Big|~ \bq^\VT(\pi^\AV)\right],
\]
as desired. 
\end{proof}


\subsection{Proof of Corollary \ref{coro:c1}}
\begin{corollary}
The enhanced imitation learning model \eqref{prob:E-GAIL} is equivalent to GAIL with the modified discriminator objective:
\begin{align}
\max_{\widetilde{\pi}^{\VT}} \; & \min_{D\in [0,1]}  \Bigg\{\phi^E(\widetilde{\pi}^{\VT},D|\pi^\AV) = \bbE_{\tau \sim \widetilde{\pi}^{\VT}} \left[\sum_{t} \eta(s_t, a^{\AV}_t, a^{\VT}_t)  ~\big| \bq^\VT(\pi^{\AV})\right]+ \bbE_{\tau \sim {\pi}^{\VT}} \left[\sum_{t} \log (1 - D(s_t,a^{\VT}_t)) ~\big| \bq^\VT(\pi^{\AV})\right]  - \lambda H(\widetilde{\pi}^{\VT}) \Bigg\}
\end{align}
\end{corollary}
\begin{proof}
The corollary can be deduced from Proposition \ref{prop:IM-gradient}, or one can write 
\begin{align}
V_{\pi^{\AV}}(s_0|\bq^\AV(\widetilde{\pi}^{\VT}_\psi)) = \bbE_{\tau\sim \pi^\AV} \left[ \sum_{t} \gamma^t r^\AV(s_t) \Big| \bq^\AV(\widetilde{\pi}^\VT_\psi) \right] &= \sum_{\tau} R^\AV(\tau)\prod_t \pi^\AV(a^\AV_t|s_t) P(s_{t+1}|s_t,a^\AV)\nonumber \\
&=  \sum_{\tau = \{(s_t,a^\AV_t, a^\VT_t)\}} R^\AV(\tau)\prod_t \pi^\AV(a^\AV_t|s_t) \widetilde{\pi}^\VT_\psi(a^\VT_t|s_t)P(s_{t+1}|s_t,a^\AV_t,a^\VT_t) \nonumber \\
&= \bbE_{\tau\sim \widetilde{\pi}^\VT} \left[ \sum_{t} \gamma^t r^\AV(s_t) \Big| \bq^\VT({\pi}^\AV) \right],\nonumber
\end{align}
which directly leads the desired equivalence. 
\end{proof}

\subsection{Proof of Lemma \ref{lm:lm2}}
\begin{lemma}
Given two adversary policies $\pi^\AV$ and $\widetilde{\pi}^\AV$, let  $\cH = \max_{s} \left\{\left|V_{\pi^\VT}({s}|~\bq^\VT(\pi^{\AV}))\right|\right\} $
\[
\left|\Gamma(\widetilde{\pi}^{\AV})  - \Gamma(\pi^{\AV}) \right| \leq   \frac{\gamma \cH\sqrt{2\ln 2}}{1-\gamma}  \max_{s \in S} \left\{\sqrt{D_\KL(\pi^\AV (\cdot|s)||\widetilde{\pi}^\AV (\cdot|s))} \right\}
\]
where $D_\KL(\pi^\AV (\cdot|s)||\widetilde{\pi}^\AV (\cdot|s))$ is the KL divergence between $\widetilde{\pi}^\AV$ and  ${\pi}^\AV$.
\end{lemma}
\begin{proof}
Recall that we define  $\Gamma(\pi^\AV)$ as  the  expected reward of the victim with policy $\pi^\VT$ when the adversary follows policy $\pi^\AV$
\[
\Gamma(\pi^\AV) = \bbE_{\tau \sim \pi^\VT} \left[\sum_{t} \gamma^t r^\VT(s_t) \Big| \bq^\VT(\pi^{\AV})\right].  
\]
Given two adversary policies $\pi^\AV$ and $\widetilde{\pi}^{\AV}$, we define the following victim's \textit{competitive  advantage function} $A_{\pi^{\AV}}(s,\overline{s})$ for two states  $s,\overline{s} \in \cS$
\[
A_{\pi^{\AV}}(s,\overline{s}) = r^\VT(s)+ \gamma V_{\pi^\VT}(\overline{s}|~\bq^\VT(\pi^{\AV})) - V_{\pi^\VT}({s}|~\bq^\VT(\pi^{\AV})),
\]
where $V_{\pi^\VT}({s}|~\bq^\VT(\pi^{\AV})) =\bbE_{\tau \sim \pi^\VT} \left[\sum_{t} \gamma^t r^\VT(s_t) \Big| \bq(\pi^{\AV}), s_0 = s\right].$
We then compute the advantage of the adversary policy $\pi^{\AV}$ over $\widetilde{\pi}^{\AV}$ but in terms of victim's expected rewards  as follows
\begin{align}
    \Gamma(\widetilde{\pi}^{\AV})  - \Gamma(\pi^{\AV}) &=  \bbE_{\tau \sim \pi^\VT} \left[\sum_{t} \gamma^t r^\VT(s_t) \Big| \bq^\VT(\widetilde{\pi}^{\AV})\right]  -  V_{\pi^\VT}(s_0|~\bq^\VT(\pi^{\AV})) \nonumber \\
    & \stackrel{(a)}{=} \bbE_{\tau \sim \pi^\VT, \bq^\VT(\widetilde{\pi}^{\AV})} \left[\sum_{t} \gamma^t r^\VT(s_t)\right]  + \bbE_{\tau \sim \pi^{\VT}, \bq^\VT(\widetilde{\pi}^{\AV})} \left[\sum_{t}\gamma^t( \gamma V_{\pi^\VT}(s_{t+1}|~\bq^\VT(\pi^{\AV})) - V_{\pi^\VT}(s_{t}|~\bq^\VT(\pi^{\AV})))\right] \nonumber \\
     &=   \bbE_{\tau \sim \pi^{\VT}, \bq^\VT(\widetilde{\pi}^{\AV})} \left[\sum_{t}\gamma^t\Big( r^\VT(s_t) +\gamma V_{\pi^\VT}(s_{t+1}|~\bq^\VT(\pi^{\AV})) - V_{\pi^\VT}(s_{t}|~\bq^\VT(\pi^{\AV}))\Big)\right] \nonumber\\
     & \stackrel{(b)}{=}  \bbE_{\tau \sim \pi^{\VT}, \bq^\VT(\widetilde{\pi}^{\AV})} \left[\sum_{t}\gamma^t\Big(A_{\pi^\AV} (s_t,s_{t+1})\Big)\right],\label{eq:lm2-proof-eq1}
\end{align}
where $(a)$ is due to the fact that 
\[
\sum_{t}\gamma^t( \gamma V_{\pi^\VT}(s_{t+1}|~\bq^\VT(\pi^{\AV})) - V_{\pi^\VT}(s_{t}|~\bq^\VT(\pi^{\AV})))  = V_{\pi^\VT}(s_0|~\bq^\VT(\pi^{\AV})), 
\]
and $(b)$ is due to the definition of the competitive advantage function $A_{\pi^\AV} (s_t,s_{t+1})$. 
Here we note that 
\[
\bbE_{\overline{s}^{\VT} \sim \pi^\VT,\bq(\pi^\AV)|s^\VT}\Big[A_{\pi^{\AV}}(s^{\VT},\overline{s}^{\VT})\Big] = \bbE_{\overline{s}^{\VT} \sim \pi^\VT,\bq(\pi^\AV)|s^\VT}\Big[r(s^\VT)+ \gamma V_{\pi^\VT}({s}^{\VT}|~\bq(\pi^{\AV})) - V_{\pi^\VT}(\overline{s}^{\VT}|~\bq(\pi^{\AV}))\Big] = 0. 
\]
We further have the following bound for the competitive advantage function. 
\begin{align}
\bbE_{\overline{s} \sim \pi^\VT,\bq^\VT(\widetilde{\pi}^\AV)|s}\Big[A_{\pi^{\AV}}(s,\overline{s})\Big] &=  \bbE_{\overline{s} \sim \pi^\VT, \bq^\VT(\widetilde{\pi}^{\AV})}\left[ r^\VT(s)+ \gamma V_{\pi^\VT}(\overline{s}|~\bq^\VT(\pi^{\AV}))\right] - \bbE_{\overline{s} \sim \pi^\VT, \bq^\VT({\pi}^{\AV})}\left[ r^\VT(s)+ \gamma V_{\pi^\VT}(\overline{s}|~\bq^\VT(\pi^{\AV}))\right] \nonumber \\
&=\gamma \left(\sum_{\overline{s}} V_{\pi^\VT}(\overline{s}|~\bq^\VT(\pi^{\AV})) \Big(P(\overline{s}|s,\pi^\VT,\widetilde{\pi}^\AV) - P(\overline{s}|s,\pi^\VT,\pi^\AV)\Big)\right)\nonumber \\
&\leq \gamma \cH \bbE_{\overline{s} \sim \pi^\VT| s}\left[ ||\pi^\AV (\cdot|\overline{s}) -\pi^\AV (\cdot|\overline{s})||_1\right] \nonumber  \\
&\stackrel{(c)}{\leq} \gamma \cH ~ \max_{s \in \cS} \left\{\sqrt{2\ln 2 \KL(\pi^\AV (\cdot|s)||\widetilde{\pi}^\AV (\cdot|s))} \right\},\label{eq:lm2-proof-eq2}
\end{align}
where $\cH = \max_{s}\{|V_{\pi^\VT} (s|\bq^\VT(\pi^\AV))|\}$ and $(c)$ is due to the inequality $||p-q||_1 \leq \sqrt{2\ln 2 D_\KL (p||q)}$ for two distributions $p,q$ \citep{thomas2006elements}.
Moreover,  if we define
\[\epsilon = 
\max_{s}\left\{\Big|\bbE_{\overline{s} \sim \pi^\VT,\bq(\widetilde{\pi}^\AV)|s}\Big[A_{\pi^{\AV}}(s,\overline{s})\Big]\Big|\right\}, 
\]
then $\epsilon \rightarrow 0$ if $\widetilde{\pi}^{\AV} \rightarrow {\pi}^{\AV}$. Moreover,  we can see from  \eqref{eq:lm2-proof-eq1} that 
\begin{equation}\label{eq:lm2-proof-eq3}
\left|\Gamma(\widetilde{\pi}^{\AV})  - \Gamma(\pi^{\AV}) \right|  =   \bbE_{\tau \sim \pi^{\VT}, \bq^\VT(\widetilde{\pi}^{\AV})} \left[\sum_{t}\gamma^t\Big(A_{\pi^\AV} (s_t,s_{t+1})\Big)\right]\leq  \bbE_{\tau \sim \pi^{\VT}, \bq(\widetilde{\pi}^{\AV})} \left[\sum_{t=0}^{\infty}\gamma^t\epsilon\right] = \frac{\epsilon}{1-\gamma}.
\end{equation}
Putting \eqref{eq:lm2-proof-eq2} and \eqref{eq:lm2-proof-eq3} together, we can bound the gap $\left|\Gamma(\widetilde{\pi}^{\AV})  - \Gamma(\pi^{\AV}) \right|$ as  
\begin{align}
     \left|\Gamma(\widetilde{\pi}^{\AV})  - \Gamma(\pi^{\AV})\right| &\leq \frac{\max_{s}\left\{\Big|\bbE_{\overline{s} \sim \pi^\VT,\bq(\widetilde{\pi}^\AV)|s}\Big[A_{\pi^{\AV}}(s,\overline{s})\Big]\Big|\right\}}{1-\gamma} \nonumber \\
    &\leq  \frac{\gamma \cH}{1-\gamma}
 ~ \max_{s \in \cS} \left\{\sqrt{2\ln 2 \KL(\pi^\AV (\cdot|s)||\widetilde{\pi}^\AV (\cdot|s))} \right\},\nonumber
\end{align}
as desired. 
\end{proof}

\subsection{Proof of Theorem \ref{th:th1}}
\begin{theorem}
Suppose that  discriminator's network model $D$ of \eqref{prob:E-GAIL} varies within $[D^L,D^U] \subset [0,1]$. Let $\pi^{\AV*}$ be the target adversary policy that we want to train the imitation policy with, and let $(\widetilde{\pi}^{\VT*}, D^{\VT*}) $ be the imitation policy and the imitator's discriminator trained with another adversary $\pi^\AV$, we have the following performance guarantee for $\widetilde{\pi}^{\VT*}$.
\begin{align}
\Big| \phi^E(\widetilde{\pi}^{\VT*},D^{\VT*}|{\pi}^{\AV*})& - \max_{\widetilde{\pi}^{\VT}}\min_{D}\{\phi^E(\widetilde{\pi}^{\VT},D|{\pi}^{\AV*})\}\Big| \leq 2K \max_{s \in S} \left\{\sqrt{D_\KL(\pi^\AV (\cdot|s)||{\pi}^{\AV*} (\cdot|s))} \right\},
\end{align}
where 
\[
K = \frac{\gamma \sqrt{2\ln 2} \left(\max_{s}\{r^\VT(s)\} - \log (D^L-D^L D^U)\right)}{(1-\gamma)^2}.
\]
\end{theorem}
\begin{proof}
From the proof of Lemma \ref{lm:lm2} we can deduce the following, for any policies $\pi^\AV$, $\widetilde{\pi}^\AV$, and any reward function $r^\VT(a)$, 
\begin{align}
    &\left|\bbE_{\tau \sim \pi^\VT} \left[\sum_{t} \gamma^t r^\VT(s_t) \Big| \bq^\VT(\pi^{\AV})\right]  -  \bbE_{\tau \sim \pi^\VT} \left[\sum_{t} \gamma^t r^\VT(s_t) \Big| \bq^\VT(\widetilde{\pi}^{\AV})\right]  \right|  \nonumber  \\ 
    & \qquad \qquad \leq  \frac{\gamma}{1-\gamma} \max_{s}\{|V_{\pi^\VT} (s|\bq^\VT(\pi^\AV))|\}  \max_{s \in \cS} \left\{\sqrt{2\ln 2 \KL(\pi^\AV (\cdot|s)||\widetilde{\pi}^\AV (\cdot|s))} \right\} \nonumber \\
    &\qquad\qquad \leq \frac{\gamma \max_s |r^\VT(s)|}{(1-\gamma)^2} \left\{\sqrt{2\ln 2 \KL(\pi^\AV (\cdot|s)||\widetilde{\pi}^\AV (\cdot|s))} \right\}. \label{eq:this-for-theorem2}
\end{align}
We now using this to bound the
gap $  \left|\phi^E(\widetilde{\pi}^{\VT},D|\widetilde{\pi}^\AV) -  \phi^E(\widetilde{\pi}^{\VT},D|\pi^\AV)\right|$ as follows. We first write 
\begin{align}
    &\left|\phi^E(\widetilde{\pi}^{\VT},D|\widetilde{\pi}^\AV) -  \phi^E(\widetilde{\pi}^{\VT},D|\pi^\AV)\right|\leq \left|\bbE_{\tau \sim \widetilde{\pi}^{\VT}} \left[\sum_{t} \gamma^t(\log (D) - r(s_t))  ~\big| \bq^\VT(\widetilde{\pi}^{\AV})\right] - \bbE_{\tau \sim \widetilde{\pi}^{\VT}} \left[\sum_{t} \gamma^t(\log (D) - r(s_t))  ~\big| \bq^\VT(\pi^{\AV})\right] \right| + \nonumber \\
    &\qquad\qquad\qquad+ \left|\bbE_{\tau \sim {\pi}^{\VT}} \left[\sum_{t} \gamma^t\log (1 - D) ~\big| \bq^\VT(\widetilde{\pi}^{\AV})\right] - \bbE_{\tau \sim {\pi}^{\VT}} \left[\sum_{t} \gamma^t\log (1 - D) ~\big| \bq^\VT(\pi^{\AV})\right]\right|\nonumber \\
    &\leq  \frac{\gamma \sqrt{2\ln 2}}{(1-\gamma)^2}  \max_{s} \left\{\sqrt{\KL(\pi^\AV (\cdot|s)||\widetilde{\pi}^\AV (\cdot|s))} \right\} \left(\max_{s,D}\{|r^\AV(s) - \log (D)\} + \max_D |\log(1-D)|\right)\nonumber\\
     &\stackrel{(d)}{\leq}  \frac{\gamma \sqrt{2\ln 2}}{(1-\gamma)^2}  \max_{s} \left\{\sqrt{\KL(\pi^\AV (\cdot|s)||\widetilde{\pi}^\AV (\cdot|s))} \right\} \left(\max_{s}\{r(s^{\AV})\} - \log (D^L) - \log(1-D^U)\right),\label{eq:gap-phiE}
\end{align}
where $(d)$ is because $D\in [D^L,D^U]$.
For ease of notation, let 
\[
K = \frac{\gamma \sqrt{2\ln 2} \left(\max_{s^{\AV}}\{r(s^{\AV})\} - \log (D^L) - \log(1-D^U)\right)}{(1-\gamma)^2};~~\epsilon = \sqrt{D_\KL(\pi^\AV (\cdot|s)||\widetilde{\pi}^\AV (\cdot|s))}.
\]
We first try to bound  $\left|\min_{D}\{\phi^E(\widetilde{\pi}^{\VT},D|{\pi}^\AV)\}  - \min_{D}\{\phi^E(\widetilde{\pi}^{\VT},D|{\pi}^{\AV*})\}  \right|$ as follows.  
\begin{itemize}
    \item If $\min_{D}\{\phi^E(\widetilde{\pi}^{\VT},D|{\pi}^\AV)\}  \geq  \min_{D}\{\phi^E(\widetilde{\pi}^{\VT},D|{\pi}^{\AV*})\}$, then we let $D^* = \text{argmin}_D \{\phi^E(\widetilde{\pi}^{\VT},D|{\pi}^{\AV*})\}$ to have
    \begin{align}
        \left|\min_{D}\{\phi^E(\widetilde{\pi}^{\VT},D|{\pi}^\AV)\}  - \min_{D}\{\phi^E(\widetilde{\pi}^{\VT},D|{\pi}^{\AV*})\}  \right| &=\min_{D}\{\phi^E(\widetilde{\pi}^{\VT},D|{\pi}^\AV)\}  - \min_{D}\{\phi^E(\widetilde{\pi}^{\VT},D|{\pi}^{\AV*})\} \nonumber \\
        &\leq \phi^E(\widetilde{\pi}^{\VT},D^*|{\pi}^\AV) - \phi^E(\widetilde{\pi}^\VT,D^*|{\pi}^{\AV*})\nonumber\\
        &\leq K\epsilon
    \end{align}
    \item If $\min_{D}\{\phi^E(\widetilde{\pi}^{\VT},D|{\pi}^\AV)\}  \leq  \min_{D}\{\phi^E(\widetilde{\pi}^{\VT},D|{\pi}^{\AV*})\}$, then we let $D^* = \text{argmin}_D \{\phi^E(\widetilde{\pi}^{\VT},D|{\pi}^{\AV})\}$  to have a similar evaluation
      \begin{align}
         \left|\min_{D}\{\phi^E(\widetilde{\pi}^{\VT},D|{\pi}^\AV)\}  - \min_{D}\{\phi^E(\widetilde{\pi}^{\VT},D|{\pi}^{\AV*})\}  \right| &= \min_{D}\{\phi^E(\widetilde{\pi}^{\VT},D|{\pi}^{\AV*})\}  - \min_{D}\{\phi^E(\widetilde{\pi}^{\VT},D|{\pi}^\AV)\} \nonumber \\
        &\leq \phi^E(\widetilde{\pi}^{\VT},D^*|{\pi}^{\AV*}) - \phi^E(\widetilde{\pi}^\VT,D^*|{\pi}^{\AV})\nonumber\\
        &\leq K\epsilon.
        \end{align}
\end{itemize}
So we always have 
\begin{equation}
\label{eq:eq3211}
   \left|\min_{D}\{\phi^E(\widetilde{\pi}^{\VT},D|{\pi}^\AV)\}  - \min_{D}\{\phi^E(\widetilde{\pi}^{\VT},D|{\pi}^{\AV*})\}  \right| \leq K\epsilon. 
\end{equation}
Now, suppose $\pi^{\AV*}$ is a target adversary policy that we want the imitation learning model to train with,  and let $(\widetilde{\pi}^{\VT*}, D^{{\VT*}})$ be an imitation learning policy and the discriminator network that are trained with adversary policy $\pi^{\AV}$. 
To bound the gap between $| \phi^E(\widetilde{\pi}^{\VT*},D^{\VT*}|{\pi}^{\AV*}) - \max_{\widetilde{\pi}^{\VT}}\min_{D}\phi^E(\widetilde{\pi}^{\VT},D|{\pi}^{\AV*})|$, we consider the following two cases 
\begin{itemize}
    \item If $\phi^E(\widetilde{\pi}^{\VT*},D^{\VT*}|{\pi}^{\AV*}) \geq  \max_{\widetilde{\pi}^{\VT}}\min_{D}\phi^E(\widetilde{\pi}^{\VT},D|{\pi}^{\AV*})$, then 
\begin{align}
    &\left| \phi^E(\widetilde{\pi}^{\VT*},D^{\VT*}|{\pi}^{\AV*}) -\max_{\widetilde{\pi}^{\VT}}\min_{D}\phi^E(\widetilde{\pi}^{\VT},D|{\pi}^{\AV*}) \right|=  \phi^E(\widetilde{\pi}^{\VT*},D^{\VT*}|{\pi}^{\AV*}) -\max_{\widetilde{\pi}^{\VT}}\min_{D}\phi^E(\widetilde{\pi}^{\VT},D|{\pi}^{\AV*}) \nonumber \\
    &\stackrel{(e)}{\leq} K\epsilon +  \phi^E(\widetilde{\pi}^{\VT*},D^{\VT*}|{\pi}^{\AV}) -\min_{D}\phi^E(\widetilde{\pi}^{\VT*},D|{\pi}^{\AV*}) \nonumber \\
    &\stackrel{(f)}{ = } K\epsilon +  \min_D \phi^E(\widetilde{\pi}^{\VT*},D|{\pi}^{\AV}) - \min_{D}\phi^E(\widetilde{\pi}^{\VT*},D|{\pi}^{\AV*}) \nonumber \\
    &{\leq} K\epsilon + \left|\min_{D}\{\phi^E(\widetilde{\pi}^{\VT*},D|{\pi}^\AV)\}  - \min_{D}\{\phi^E(\widetilde{\pi}^{\VT*},D|{\pi}^{\AV*})\}  \right| \nonumber \\
    &\stackrel{(g)}{\leq} 2K \epsilon,
\end{align}
where $(e)$ is because  $|\phi^E(\widetilde{\pi}^{\VT*},D|{\pi}^{\AV*}) - \phi^E(\widetilde{\pi}^{\VT*},D|{\pi}^{\AV})|\leq K\epsilon$ (according to \eqref{eq:gap-phiE}), $(f)$ is due to $D^{\VT*} = \text{argmax}_D \phi^E(\widetilde{\pi}^{\VT*},D|{\pi}^{\AV})$ and $(g)$ is because of 
\eqref{eq:eq3211}.
\item If $\phi^E(\widetilde{\pi}^{\VT*},D^{\VT*}|{\pi}^{\AV*}) \geq  \max_{\widetilde{\pi}^{\VT}}\min_{D}\phi^E(\widetilde{\pi}^{\VT},D|{\pi}^{\AV*})$, we let $\widetilde{\pi}^{\VT**} = \text{argmax}_{\widetilde{\pi}^{\VT}}\min_{D}\phi^E(\widetilde{\pi}^{\VT},D|{\pi}^{\AV*})$ and write
\begin{align}
    &\left| \phi^E(\widetilde{\pi}^{\VT*},D^{\VT*}|{\pi}^{\AV*}) -\max_{\widetilde{\pi}^{\VT}}\min_{D}\phi^E(\widetilde{\pi}^{\VT},D|{\pi}^{\AV*}) \right|=  \max_{\widetilde{\pi}^{\VT}}\min_{D}\phi^E(\widetilde{\pi}^{\VT},D|{\pi}^{\AV*}) - \phi^E(\widetilde{\pi}^{\VT*},D^{\VT*}|{\pi}^{\AV*}) \nonumber \\
    &\stackrel{(h)}{\leq}  \min_{D}\phi^E(\widetilde{\pi}^{\VT**},D|{\pi}^{\AV*})     - \phi^E(\widetilde{\pi}^{\VT*},D^{\VT*}|{\pi}^{\AV}) + K\epsilon \nonumber \\
    &\stackrel{(i)}{\leq }  K\epsilon +  \min_D \phi^E(\widetilde{\pi}^{\VT**},D|{\pi}^{\AV*}) - \min_{D}\phi^E(\widetilde{\pi}^{\VT**},D|{\pi}^{\AV}) \nonumber \\
    &{\leq} K\epsilon + \left| \min_D \phi^E(\widetilde{\pi}^{\VT**},D|{\pi}^{\AV*}) - \min_{D}\phi^E(\widetilde{\pi}^{\VT**},D|{\pi}^{\AV})   \right| \nonumber \\
    &\stackrel{(j)}{\leq} 2K \epsilon,
\end{align}
where $(h)$ is due to $ \phi^E(\widetilde{\pi}^{\VT*},D^{\VT*}|{\pi}^{\AV})-\phi^E(\widetilde{\pi}^{\VT*},D^{\VT*}|{\pi}^{\AV*})\leq K\epsilon$ (see \eqref{eq:gap-phiE}), $(i)$ is due to $\phi^E(\widetilde{\pi}^{\VT*},D^{\VT*}|{\pi}^{\AV}) = \max_{\widetilde{\pi}^{\VT}}\min_D \phi^E(\widetilde{\pi}^{\VT},D|{\pi}^{\AV}) \geq  \min_D \phi^E(\widetilde{\pi}^{\VT**},D|{\pi}^{\AV}) $, and $(j)$ is because of 
\eqref{eq:eq3211}.
\end{itemize}
Combine the two cases above we obtain he desired bound. 
\[
|\phi^E(\widetilde{\pi}^{\VT*},D^{\VT*}|{\pi}^{\AV*}) - \max_{\widetilde{\pi}^{\VT}}\min_{D}\phi^E(\widetilde{\pi}^{\VT},D|{\pi}^{\AV*})| \leq 2K \max_{s \in \cS} \left\{\sqrt{\KL(\pi^\AV (\cdot|s)||\widetilde{\pi}^\AV (\cdot|s))} \right\}
\].
\end{proof}


\subsection{Proof of Proposition \ref{prop:EAIL} }
\begin{proposition}
The gradient of \eqref{prob:paper2} w.r.t  adversary's policy can be computed as follows:
\[
\nabla_{\theta} \left( V_{\pi^{\AV}_\theta}(s_0) - V_{\pi^{\VT}}(s_0|\bq^\VT(\pi^{\AV}_\theta))\right) 
= \bbE_{\tau\sim (\pi^{\VT}, \pi^{\AV})} \left[  \Delta^R(\tau)\sum_{t} \nabla_{\theta} \log \pi^{\AV}_\theta(a^{\AV}_t|s_t) \right]
\]
where $\Delta^R(\tau) =\sum_{t}\gamma^t ( r^{\AV}(s_t)- r^{\VT}(s_t))$.
\end{proposition}
\begin{proof}
Similarly to the proof of Lemma \ref{lm:lm1}, we compute the gradient of $V_{\pi^\VT}(\cdot)$ as
\begin{align}
    \nabla_\theta \left(V_{\pi^{\VT}}(s_0|\bq^\VT(\pi^{\AV}_\theta))\right) &=\sum_{\tau = \{(s_t,a^\AV_t, a^\VT_t)\}\sim \pi^\AV, \pi^\VT} \left(\sum_t \gamma^t r^\VT(s_t)\right)P(\tau) \sum_t \nabla_\theta\log{\pi}^\AV_\theta(a^\AV_t|s_t) \nonumber \\
&= \bbE_{\tau = \{(s_t,a^\AV_t, a^\VT_t)\}\sim \pi^\AV, \pi^\VT} \left[\left(\sum_t \gamma^t r^\VT(s_t)\right) \sum_t \nabla_\psi\log{\pi}^\AV_\theta(a^\AV_t|s_t) \right],\nonumber  
\end{align}
Thus, we can write the gradient of \eqref{prob:paper2} as
\[
\nabla_{\theta} \left( V_{\pi^{\AV}_\theta}(s_0) - V_{\pi^{\VT}}(s_0|\bq^\VT(\pi^{\AV}_\theta))\right) 
= \bbE_{\tau\sim (\pi^{\VT}, \pi^{\AV})} \left[  \left(\sum_t \gamma^t (r^\AV(s_t) - r^\VT(s_t))\right)\sum_{t} \nabla_{\theta} \log \pi^{\AV}_\theta(a^{\AV}_t|s_t) \right],
\]
which concludes the proof. 
\end{proof}

\subsection{Proof of Corollary \ref{coro:c2}}
\begin{corollary}
\eqref{prob:paper2} is equivalent to
\[
 \max_{\pi^{\AV}} \left\{ \bE_{\tau \sim \pi^{\AV}}\left[\sum_{t=0}^\infty \gamma^t \Delta^r(s_t)\;\Big| \bq^\AV(\pi^{\VT}) \right] \right\},
\]
where $\Delta^r(s_t) =  r^{\AV}(s_t)- r^{\VT}(s_t)$.
\end{corollary}
\begin{proof}
The equivalence can be straightforwardly deduced from Proposition \ref{prop:EAIL}.
\end{proof}

\subsection{Proof of Theorem \ref{th:th2}}
\begin{theorem}
For any $\epsilon>0$, we have the following bound
\[
\left|Y(\epsilon) - Y^* \right| \leq  \frac{\gamma\sqrt{2\ln 2} \max_s\{ |\Delta^r(s)|\}}{(1-\gamma)^2} \sqrt{\epsilon}. 
\]
\end{theorem}
\begin{proof}
In analogy to the proofs of Lemma \ref{lm:lm2} and Theorem \ref{th:th1}, if we define $\Lambda(\pi^\VT)$ as the adversary's expected return if the victim's policy is $\pi^\VT$
\[
\Lambda(\pi^\VT) = \max_{\pi^{\AV}} \left\{ \bE_{\tau \sim \pi^{\AV}}\left[\sum_{t=0}^\infty \gamma^t \Delta^r(s_t)\;\Big| \bq^\AV(\pi^{\VT}) \right] \right\}
\]
then, similarly to the derivation  in \eqref{eq:this-for-theorem2}, we can get the following bound for any $\pi^\VT \in\Omega(\epsilon)$
\begin{align}
\left|\Lambda(\pi^\VT) - \Lambda(\pi_0^\VT) \right| &\leq \frac{\gamma \max_s |\Delta^r(s)|}{(1-\gamma)^2} \left\{\sqrt{2\ln 2 \KL(\pi^\VT (\cdot|s)||\widetilde{\pi}^\VT (\cdot|s))} \right\}\nonumber \\
&\leq \frac{\sqrt{2\ln 2 \epsilon }\gamma \max_s |\Delta^r(s)|}{(1-\gamma)^2} 
\end{align}
which also implies that
\[
\left|\min_{\pi^\VT \in \Omega(\epsilon)}\{\Lambda(\pi^\VT)\} - Y^* \right| \leq \frac{\sqrt{2\ln 2 \epsilon }\gamma \max_s |\Delta^r(s)|}{(1-\gamma)^2},
\]
which is also the desired result. 
\end{proof}

\section{Adversarial Policy Imitation Learning Algorithm}
Algorithm \ref{algo:main-EAPIL} shows the detailed steps of our Adversarial Policy Imitation Learning algorithms and Figure \ref{fig:train-attacker} provides
a more detailed overview  of our framework.

\begin{algorithm*}[htb]
\caption{Adversarial Policy Imitation Learning}
\label{algo:main-EAPIL}
\begin{algorithmic}
\State \textbf{Input:} Environment $env$; adversary policy $\pi^\AV_\theta$; defender policy $\pi^\VT$; imitator defender policy $\widetilde{\pi}^\VT_\psi$; discriminator $D_w$; adversary replay buffer $\mathcal{D}^\AV$; imitator replay buffer $\widetilde{\mathcal{D}}^\VT$; expert trajectory buffer $\mathcal{D}^\VT_\tau$; imitator trajectory buffer $\widetilde{\mathcal{D}}^\VT_\tau$; initial trainable parameters $\theta_0, \psi_0, w_0$.

\While{repeat a certain number of times}
\State $s_0 \gets env.\text{reset}()$ \Comment{Start new episode}
\For{$t=0,1,2,...$}
  \State \comments{Sampling}
  \State Sample victim action $a_t^\VT$ and imitator action $\widetilde{a}_t^\VT$ by $\pi^\VT (s_t)$ and $\widetilde{\pi}^\VT_\psi (s_t)$.
  \State Sample adversary action $a_t^\AV \sim \pi^\AV_\theta (s_t')$, where $s_t' = (s_t, \widetilde{a}_t^\VT)$. \Comment{Sample actions}
  \State $(s_{t+1}, r_t, d_t) \gets env.\text{step}(a_t^\VT, a_t^\AV)$ \Comment{Next step}
  \State $\widetilde{r}_t^\VT \gets \eta(s_t, a^{\AV}_t, a^{\VT}_t)$
  \State Append transition $(s_t', a_t^\AV, r_t^\AV)$, $(s_t, \widetilde{a}_t^\VT, \widetilde{r}_t^\VT)$, $(s_t, a_t^\VT)$ and $(s_t, \widetilde{a}_t^\VT)$ respectively to $\mathcal{D}^\AV$, $\widetilde{\mathcal{D}}^\VT$, $\mathcal{D}^\VT_\tau$ and $\widetilde{\mathcal{D}}^\VT_\tau$.
  \If{any buffer is full} \Comment{Update parameters}
    \For{$i=0,1,2,...$}
      \State Sample trajectories from $\mathcal{D}^\AV$, $\widetilde{\mathcal{D}}^\VT$, $\mathcal{D}^\VT_\tau$ and $\widetilde{\mathcal{D}}^\VT_\tau$
      \State \comments{Updating imitator's discriminator parameters}
      \State Update imitator's discriminator parameters from $w_i$ to $w_{i+1}$ using \eqref{eq:update-Dw}.
      \State Take an imitator's generator policy step from $\psi_i$ to $\psi_{i+1}$ using TRPO rule using gradients in \eqref{eq:update-psi}
      \State \comments{Updating adversary's policy network}
      \State Take an adversary policy step from $\theta_i$ to $\theta_{i+1}$ via PPO using the gradients given in \eqref{prop:EAIL}.
    \EndFor
    \State Clear all buffers and update trainable parameters $\theta_0, \psi_0, w_0$.
  \EndIf
  \If{$d_t$ is true} \Comment{Terminated}
      \State \textbf{break}
  \EndIf
\EndFor
\EndWhile
\end{algorithmic}
\end{algorithm*}

\begin{figure*}[htp]
\caption{An overview of our imitation-based adversarial policy learning (APIL) platform.}
\label{fig:train-attacker}
\centering
\vspace{0.2cm}
\includegraphics[width=0.9\textwidth]{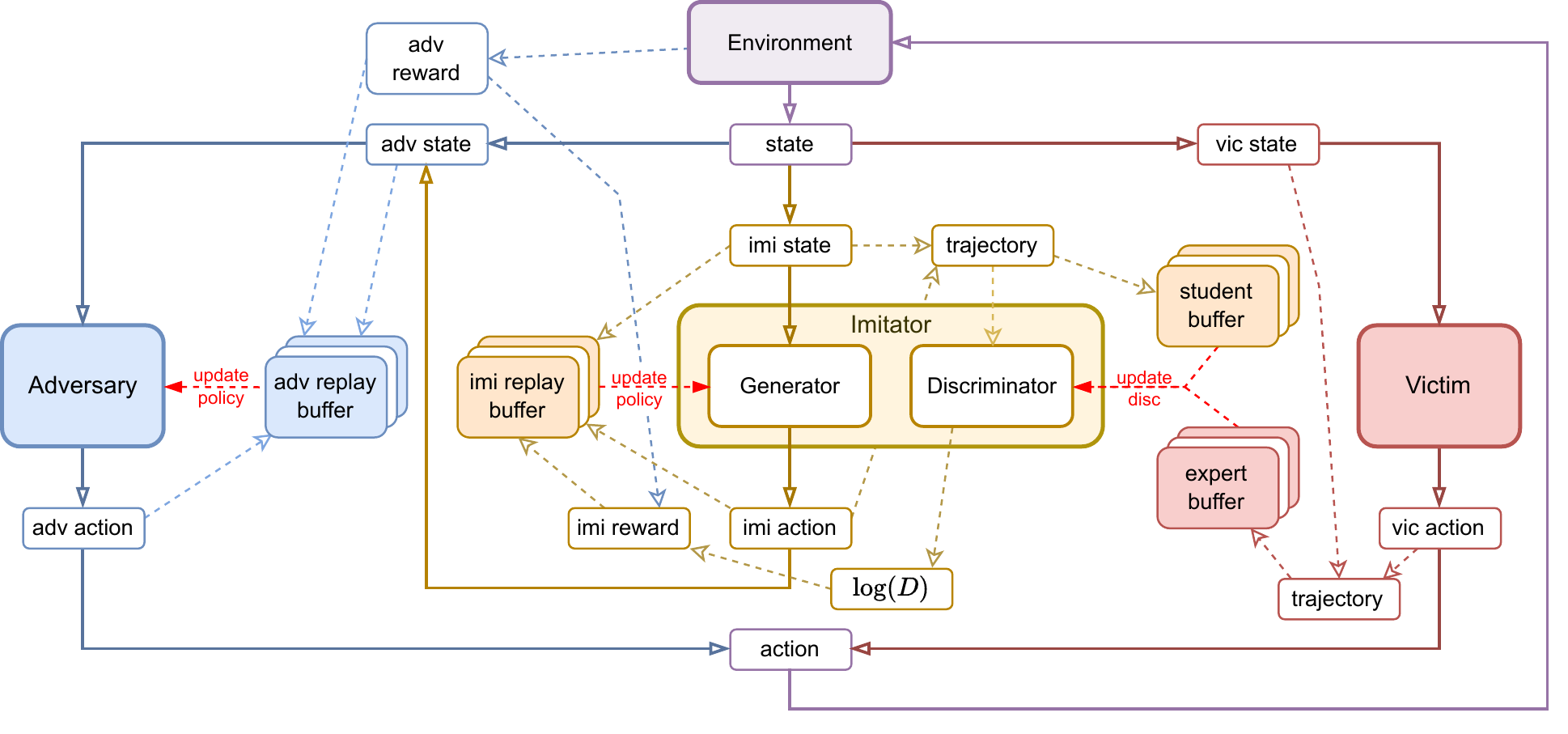}
\end{figure*}

\section{Experimental Settings and Additional Experiments}

Similar to APL's settings \citep{Guo2021AdversarialPL}, we train our adversary agents with 35M steps and retrain victim with 10M steps. In addition, we use the same model architectures, hidden dimensions and  hyperparameters with 1 GPU Nvidia GeForce RTX 2080 Ti, 24-core CPU Intel Xeon 4116 @ 2.1GHz and 64G RAM. All evaluation are tested individually with different seeds over 1000 episodes. 

When retraining the victim, we let the victim agent play with a mixing adversary which is a combination of the newly trained adversary and the baseline one. That is, we randomly taking actions from both adversaries. In Figures \ref{fig:vic-new adv} and \ref{fig:vic-old adv}, we  report the performance of the retraining victim, but with each adversary separately. It is clear that the victim retraining performance improves over episodes  for both adversaries for \textit{Kick-and-Defend, You-Shall-Not-Pass} and \textit{Sumo-Humans}. However, for \textit{Sumo-Ants},   the performance from retraining with our trained adversary improves, but that from retraining with the baseline adversary degrades.  This also indicates an advantage of our trained adversary agent in terms of retraining the victim.

\begin{figure*}[htb]
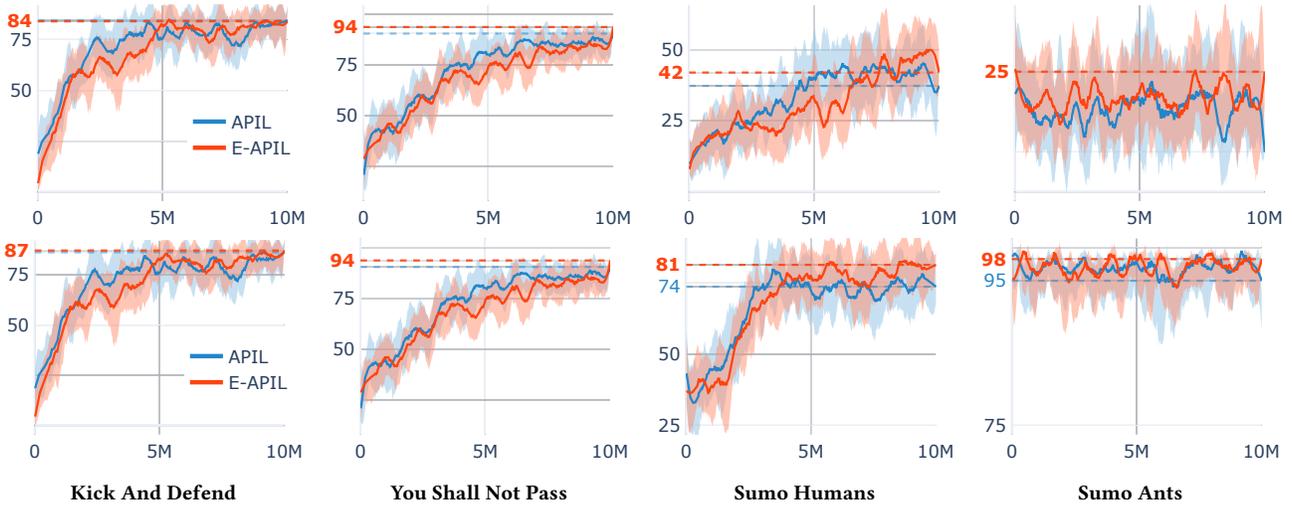

\caption{Performance of retraining victim with   trained adversary. \\ First line: win-rate. Second line: win-rate + tie-rate.}\label{fig:vic-new adv}
\centering
\showinstance[0.23]{KickAndDefend-v0}{vic_win_legend}{}
\showinstance[0.23]{YouShallNotPassHumans-v0}{vic_win}{}
\showinstance[0.23]{SumoHumans-v0}{vic_win}{}
\showinstance[0.23]{SumoAnts-v0}{vic_win}{}
\showinstance[0.23]{KickAndDefend-v0}{vic_win+tie_legend}{Kick And Defend}
\showinstance[0.23]{YouShallNotPassHumans-v0}{vic_win+tie}{You Shall Not Pass}
\showinstance[0.23]{SumoHumans-v0}{vic_win+tie}{Sumo Humans}
\showinstance[0.23]{SumoAnts-v0}{vic_win+tie}{Sumo Ants}
\end{figure*}

\begin{figure*}[htb]
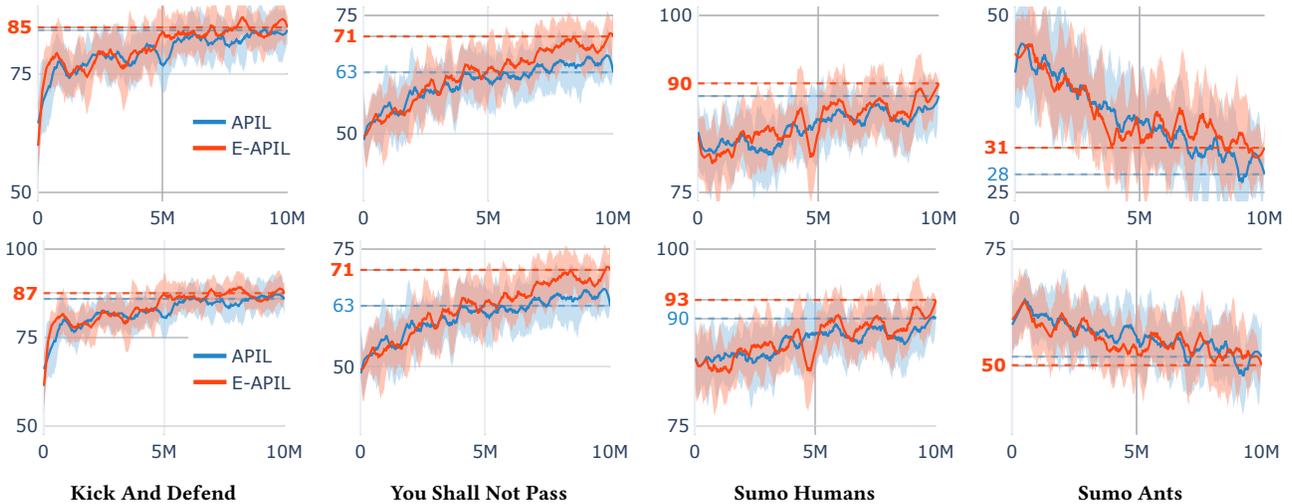

\caption{Performance of retraining victim with  baseline adversary. \\ First row: win-rate. Second row: win-rate + tie-rate.} \label{fig:vic-old adv}
\centering
\showinstance[0.23]{KickAndDefend-v0}{vic_ori_win_legend}{}
\showinstance[0.23]{YouShallNotPassHumans-v0}{vic_ori_win}{}
\showinstance[0.23]{SumoHumans-v0}{vic_ori_win}{}
\showinstance[0.23]{SumoAnts-v0}{vic_ori_win}{}
\showinstance[0.23]{KickAndDefend-v0}{vic_ori_win+tie_legend}{Kick And Defend}
\showinstance[0.23]{YouShallNotPassHumans-v0}{vic_ori_win+tie}{You Shall Not Pass}
\showinstance[0.23]{SumoHumans-v0}{vic_ori_win+tie}{Sumo Humans}
\showinstance[0.23]{SumoAnts-v0}{vic_ori_win+tie}{Sumo Ants}
\end{figure*}


\end{document}